\documentclass{article} % For LaTeX2e

\usepackage{hyperref}
\usepackage{url}

\usepackage{amsmath}
\usepackage{amsfonts}
\usepackage{amsthm}

\newtheorem{theorem}{Theorem}
\newtheorem{remark}[theorem]{Remark}%
\newtheorem{definition}[theorem]{Definition}
\newtheorem{lemma}[theorem]{Lemma}

\usepackage[round]{natbib}

\usepackage{hyperref}
\hypersetup{colorlinks=true,allcolors=[rgb]{0,0,1}}

\usepackage{algorithm}
\usepackage[noend]{algpseudocode}
\usepackage{graphicx}
\usepackage{subfigure}
\usepackage{subcaption}
\usepackage[e]{esvect}
\usepackage{xcolor}
\usepackage[margin=1in]{geometry}

\usepackage{wrapfig}

\usepackage{listings}
\definecolor{codegreen}{rgb}{0,0.6,0}
\definecolor{codegray}{rgb}{0.5,0.5,0.5}
\definecolor{codepurple}{rgb}{0.58,0,0.82}
\definecolor{backcolour}{rgb}{0.95,0.95,0.92}

\usepackage{minitoc}

\title{$k$NN Attention Demystified: A Theoretical Exploration for Scalable Transformers}

\author{Themistoklis Haris \\
Boston University \\
\texttt{tharis@bu.edu}
}

\begin{document}

\maketitle

\begin{abstract}
Despite their power, Transformers \citep{vaswani2017attention} face challenges with long sequences due to the quadratic complexity of self-attention. To address this limitation, methods like $k$-Nearest-Neighbor ($k$NN) attention have been introduced \citep{roy2021efficient}, enabling each token to attend to only its $k$ closest tokens. While $k$NN attention has shown empirical success in making Transformers more efficient, its exact approximation guarantees have not been theoretically analyzed. In this work, we establish a theoretical framework for $k$NN attention, reformulating self-attention as expectations over softmax distributions and leveraging lazy Gumbel sampling \citep{mussmann2017fast} with $k$NN indices for efficient approximation. Building on this framework, we also propose novel sub-quadratic algorithms that approximate self-attention gradients by leveraging efficient sampling techniques, such as Markov Chain-based estimation. Finally, we demonstrate the practical effectiveness of these algorithms through empirical experiments, showcasing their benefits in both training and inference.
\end{abstract}

\section{Introduction}
Transformer models have become the dominant neural architecture across language, vision, and other domains \citep{vaswani2017attention,dosovitskiy2020image}. 
However, scaling them to handle larger input sequences remains a significant challenge \citep{tay2020long}, primarily due to the quadratic complexity of computing self-attention. 
Overcoming this limitation is crucial for advancing neural networks. 
Extending context length would enable Transformers to tackle complex tasks like book summarization \citep{kryscinski2021booksum} and time-series forecasting \citep{wen2022transformers, zeng2023transformers, zhou2021informer}. 
Furthermore, improving attention efficiency would reduce the computational burden of training, making these models more accessible. 
Bridging this ``compute divide" is vital for democratizing AI \citep{ahmed2020democratization}. 

Efficient computation of self-attention has been a focal point of research 
in recent years \citep{fournier2023practical}. 
Flash Attention \citep{dao2022flashattention} and related work 
\citep{saha2024complexity} optimize the exact calculation of 
attention by minimizing wasted computation during GPU I/O operations. 
However, most approaches focus on approximating the attention function. 
Sparse Transformers improve efficiency by allowing each token to attend 
to only a small subset of tokens \citep{meister2021sparse}. 
These subsets are identified through deterministic methods 
\citep{child2019generating, guo2019star, soldaini2020cascade, li2019enhancing, 
qiu2019blockwise, beltagy2020longformer}, 
randomized algorithms \citep{kitaev2020reformer, han2023hyperattention, 
zandieh2023kdeformer, pagliardini2024fast}, 
or adaptive techniques \citep{correia2019adaptively}. 
Additionally, self-attention is often approximated using low-rank 
matrices and kernel methods \citep{wang2020linformer, tay2021synthesizer, 
xiong2021nystromformer, katharopoulos2020transformers, choromanski2020rethinking}. 
On the negative side, recent fine-grained complexity reductions indicate 
that achieving a good approximation with sub-quadratic time is not feasible across 
all scenarios \citep{keles2023computational, alman2024fast}.

In this work, we focus on sparse attention methods where each token 
vector \( q_i \in \mathbb{R}^d \) attends to the \( k \) tokens
\( k_j \in \mathbb{R}^d \) with the largest inner products \( q_i^T k_j \) 
\citep{gupta2021memory, wang2022kvt}, a paradigm we refer 
to as \textit{$k$NN Attention}. 
The Routing Transformer \citep{roy2021efficient} was an early example, 
using \( k \)-means clustering to ensure each query only attends to keys 
within the same cluster. 
Memorizing Transformers \citep{wu2022memorizing} later extended this 
approach by leveraging \( k \)NN search within a stored memory, enabling 
models to memorize new data during inference. 
More recently, Unlimiformer models \citep{bertsch2024unlimiformer} have 
improved efficiency by using a single \( k \)NN data structure 
(or \textit{index}) across all attention heads and layers.

Previous works have empirically shown that \( k \)NN Attention not only improves 
computational efficiency, but also enhances model architectures and capabilities. 
However, a rigorous theoretical analysis of \( k \)NN Attention is still lacking. 
Key questions remain unresolved, including the precise approximation guarantees 
it offers, the optimal value of \( k \), and how to extend the method 
to approximate the backward pass. 

\paragraph{Notation}
Let $Q,K,V \in \mathbb{R}^{n\times d}$ be our \textit{query, key} and \textit{value} matrices. 
Let $q_i = Q_{i,:} \in \mathbb{R}^d$ be $i$-th \textit{row} of $Q$ written as a column vector.
We will also denote the $j$-th column of $Q$ by $Q_{:,j}$.
We define $A := QK^T \in \mathbb{R}^{n\times n}$ to be the \textit{attention matrix}, and $O = \text{softmax}(A)\cdot V \in \mathbb{R}^{n\times d}$ to be the output of the \textit{attention function}. 
The softmax function is applied row-wise to $A$ and is defined as a vector valued function $\sigma:\mathbb{R}^n\to\mathbb{R}^n$:
$$
\sigma(y_1,...,y_n)_i = \frac{\exp(y_i)}{\sum_{s=1}^n \exp(y_s)}
$$ 
We also let $[n] := \{1,2,...,n\}$ and use the notation $\text{polylog}(n)$ as a substitute of $\log^k(n)$ for some arbitrary constant $k \in \mathbb{Z}^+$ that is independent of $n$. 
Finally, we use the $\widetilde{O}$ notation to hide polylogarithmic factors. We will often make use of the following boosting lemma:

\begin{lemma}[Median-Of-Means Boosting, \cite{chakrabarti2020data}]
\label{lemma:median-of-means}
If $\widehat{Q}$ is an unbiased estimator of some statistic,
then one can obtain an $(\varepsilon,\delta)$-multiplicative estimate of that statistic by suitably combining $K := \frac{C}{\varepsilon^2}\frac{\text{Var}[\widehat{Q}]}{\mathbb{E}[\widehat{Q}]^2}\ln \frac{2}{\delta}$ independent samples of $\widehat{Q}$, where $C$ is a universal constant. 
\end{lemma}

For a comprehensive outline of preliminary results and theory, please refer to Appendix \ref{section:appendix_prelims}.

\subsection{Our Contributions}
\subsubsection{A Theoretical Framework for \(k\)NN Attention}
Our work provides a theoretical framework to explain both the efficiency and 
effectiveness of \( k \)NN Attention. 
Our framework reformulates self-attention as 
expectations over softmax distributions. 
These expectations are approximated by sampling from each 
distribution in sublinear time using Lazy Gumbel Noise Sampling. 
By connecting $k$NN, $k$-Maximum Inner Product Search (MIPS), and Gumbel noise sampling, we develop a 
new sub-quadratic self-attention approximation algorithm aligning with 
the $k$NN Attention paradigm, as summarized in the following informal theorem:

\begin{theorem}
Let $Q,K,V \in \mathbb{R}^{n\times d}$ and $\varepsilon,\delta$ be positive constants. Assume $||V||_\infty = O(\log n)$. Then $k$NN-Attention as presented in Algorithm \ref{alg:knn_attn}
with $k = \sqrt{n}$ outputs a matrix $\widehat{O} \in \mathbb{R}^{n\times d}$ 
such that:
\begin{align}
    |\widehat{O}_{ij}-O_{ij}| \leq \varepsilon O_{ij}
\end{align}
for all $(i,j) \in [n]\times [d]$ with probability at least $1-\delta$ and 
in sub-quadratic time and space. 
    
\end{theorem}

\subsubsection{Approximating the Backward Pass}
Our framework can be extended to solve the problem of approximating attention gradients. 
Even though backpropagation is the main memory bottleneck for large models, 
few methods approximate attention gradients directly. 
Alman and Song's work (\citeyear{alman2024fine}) is most relevant, deriving 
inapproximability results for certain parameter regimes. 

We present new approximation algorithms for self-attention gradients using 
$k$NN search. 
A key challenge is the need to multiply by the transpose of a stochastic matrix, 
which disrupts our expectation-based reformulation. 
To address this, we use a Markov-Chain sampling technique, treating the attention 
matrix as a transition matrix and applying a single-step iteration. 
Our main theorem can be informally stated as follows:

\begin{theorem}
Let $\phi$ be a scalar loss function and 
$\partial \phi / \partial O \in \mathbb{R}^{n\times d}$. 
Then, under certain assumptions on the $||\cdot||_\infty$ norms of $Q,K,V$, 
there exist  sub-quadratic time algorithms that output estimates
$\widehat{D}^Q,\widehat{D}^K,\widehat{D}^V \in \mathbb{R}^{n\times d}$ 
for which with probability at least $1-\delta$ it holds that: 
\begin{align}
||\widehat{D}^Q-\partial \phi/\partial Q||_\infty \leq e_Q,\,
||\widehat{D}^K-\partial \phi/\partial K||_\infty
\leq e_K\,\text{ and }\,||\widehat{D}^V-\partial \phi/\partial V||_\infty \leq e_V
\end{align}
where $e_Q, e_K, e_V$ are explicit error parameters that can roughly
be bounded by $O(\varepsilon n \cdot \text{polylog}(n))$
\end{theorem}

Algorithm 
\ref{alg:dv_estimation} computes $\widehat{D}^V$, while Algorithms for $\widehat{D}^K$ and $\widehat{D}^Q$ can be found in Appendices \ref{section:estimating_dq} and \ref{section:estimating_dk}.

\section{$k$NN Attention as an Approximation Algorithm}

\subsection{Reformulating Self-Attention as Expectation}
Our first contribution is viewing the self-attention output as a matrix of 
expectations under various softmax distributions. 
Let $D_i$ be the softmax distribution defined by $D_i(j) \propto \exp(q_i^T \cdot k_j)$ over $[n]$. Then, notice that we can write:
\begin{align}
    O_{ij} &= \sum\limits_{k=1}^n \frac{\exp(q_i^T \cdot k_k)}{\sum_{s=1}^n \exp(q_i^T k_s)}\cdot V_{kj}
    =
    \sum\limits_{k=1}^n D_i(k) \cdot V_{kj}
    = \mathop{\mathbb{E}}_{k \sim D_i}\left[V_{kj}\right]
    \label{eq:expected_value_formula}
\end{align}
Thus, to approximate $O_{ij}$, we have to estimate the expected value in Equation \ref{eq:expected_value_formula}. 
Let $k$ be sampled according to $D_i$. 
Then, the estimator $\widehat{O}_{ij} = V_{kj}$ is unbiased, as $\mathop{\mathbb{E}}_{k\sim D_i}[\widehat{O}_{ij}] = O_{ij}
$. By imposing certain assumptions on the entries of the V matrix, we can bound the 
variance of $\widehat{O}_{ij}$ and use boosting to obtain explicit error 
guarantees:
\begin{theorem}
\label{thm:approx-attn}
Suppose $||V||_\infty \leq B = O(\log(n))$ and assume that for any $i \in [n]$ we can sample from $D_i$ in time $O(T)$. Then, there exists an algorithm to output a matrix $\widehat{O} \in \mathbb{R}^{n\times d}$ 
such that:
\begin{align}
    |\widehat{O}_{ij}-O_{ij}| \leq \varepsilon O_{ij}
\end{align}
for all $(i,j) \in [n]\times [d]$ with probability at least $1-\delta$, 
where $\varepsilon,\delta > 0$ are constants. The algorithm runs in $
O(nd\cdot T \cdot \varepsilon^{-2}\log(nd/\delta)\log n)
$ time.
\end{theorem}
\begin{proof}
Given that $\widehat{O}_{ij}$ is an unbiased estimator of $O_{ij}$, we can utilize Lemma \ref{lemma:median-of-means} to get an 
$(\varepsilon,\delta)$-multiplicative estimator for $O_{ij}$. 
To determine a sufficient number of samples of $\widehat{O}_{ij}$, we first bound the variance of our estimator:
\begin{align}
    \text{Var}\left[\widehat{O}_{ij}\right] \leq \mathop{\mathbb{E}}_{k \sim D_i}\left[V_{kj}^2\right] = 
    \sum\limits_{k=1}^n D_i(k)V_{kj}^2 \leq B\cdot O_{ij}
\end{align}
Then, the number of samples required is:
\begin{align}
    O\left(\varepsilon^{-2}\cdot \log(1/\delta)\cdot \text{Var}
    \left[\widehat{O}_{ij}\right]\cdot {\mathbb{E}\left[\widehat{O}_{ij}\right]^{-2}}\right) 
    =O\left(\varepsilon^{-2}\cdot \log(1/\delta)\log n\right)
\end{align}
due to our assumption on $||V||_\infty$. 
To ensure that all $nd$ elements of $O$ are approximated within the desired
guarantees, we have to set $\delta' := \delta / (nd)$ and union-bound over 
all $nd$ elements of $O$. 
Since each sample requires $O(T)$ time, we arrive at the desired time complexity.
\end{proof}

\subsection{Efficient Sampling from $D_i$ via Lazy Gumbel Sampling}
Theorem \ref{thm:approx-attn} previously assumed we could directly sample from the distribution $D_i$. The \textit{Lazy Gumbel Sampling} method proposed by \citet{mussmann2017fast} provides a way to sample from each $D_i$ in sublinear time, even with limited knowledge of $D_i$. However, there is an initial pre-processing step that takes a bit more than linear time across all the distributions.

Fix some $i \in [n]$ and let $Z_{ij} = q_i^T k_j$. In the Gumbel Max Trick (Lemma \ref{lemma:gumbel-max-trick}), 
we form the random variables $N_{ij} = Z_{ij} + G_{ij}$ where 
$G_{ij} \sim \text{Gumbel}(0,1)$ for all $j \in [n]$ and sample 
$\arg\max N_{ij}$. This is equivalent to sampling $j \in [n]$ from the softmax distribution over the $Z_{ij}$ scores.
\citet{mussmann2017fast} observed that if we have the top $k$ $Z_{ij}$ values in a set $S_i$ and add 
Gumbel noise to just them, then for any $j \notin S_i$ to be ultimately picked, its Gumbel noise $G_{ij}$ must be quite large.
We can use the concentration properties of the Gumbel distribution to 
argue that in expectation we only need to sample $\frac{n}{k}$ 
elements not in $S_i$. 
Setting $k = \sqrt{n}$ allows us to balance the two, resulting in a 
sublinear time algorithm for sampling from $D_i$. An illustration of the idea can be seen in Figure \ref{fig:lazy-gumbel}, as it was presented in \citet{mussmann2017fast}.

\begin{algorithm}
\begin{algorithmic}[1]
\caption{Lazy Gumbel Sampling from $D_i$, for some $i \in [n]$}
\label{alg:lazy-gumbel}
    \State \textbf{Inputs:} $k\in\mathbb{N}$, $q_i \in \mathbb{R}^d, K \in 
    \mathbb{R}^{n\times d}$, $S_i := \{ \text{the $k$ keys $j$ 
    with the largest $Z_{ij} := q_i^T k_j$}\}$.
    \vspace{1mm}
    \State Sample $G_{ij} \sim $ Gumbel$(0,1)$ for $j \in S_i$.
    \vspace{1mm}
    \State Let $M \gets \max\limits_{j \in S_i} \{Z_{ij} + G_{ij}\}$ and $S_{\min} \gets \min\limits_{j \in S_i} \{Z_{ij}\}$.
    \State Let $B \gets M - S_{\min}$ be the Gumbel cutoff.
    \vspace{1mm}
    \State Let $m \sim \text{Bin}(n-k, 1-\exp(-\exp(-B)))$ be the number of $
    [n]\setminus S_i$ Gumbels greater than $B$. Sample $m$ points from $[n]\setminus 
    S_i$ and denote the set of sampled points as $T_i$.
    \vspace{1mm}
    \State Sample $G_{ij} \sim \text{Gumbel}(0,1)$ conditionally greater than $B$ for each $j \in T_i$.
    \vspace{1mm}
    \State \Return $\widehat{j} \gets \arg\max\limits_{j \in S_i\cup T_i} \{Z_{ij} + 
    G_{ij}\}$
\end{algorithmic}
\end{algorithm}

We can see that this method samples exactly from $D_i$:
\begin{theorem}[Correctness of Algorithm \ref{alg:lazy-gumbel}]
After running Algorithm \ref{alg:lazy-gumbel}, it holds that:
\begin{align}
\widehat{j} = \arg\max\limits_{j \in [n]} \{q_i^T k_j + G_{ij}\}
\end{align}
where $G_{ij} \sim \text{Gumbel}(0,1)$. In other words, $\widehat{j}$ is 
sampled according to $D_i$. 
\end{theorem}
\begin{proof}
The only way that we do not find the maximum is if one of the points in $[n] 
\setminus (S_i\cup T_i)$ are the true maximum. However those points (by construction) 
have Gumbel noise at most $B$, so they cannot be the overall maximum.
\end{proof}

\begin{wrapfigure}{r}{0.5\linewidth}
    \vspace{-20pt}
    \centering
    \includegraphics[width=\linewidth]{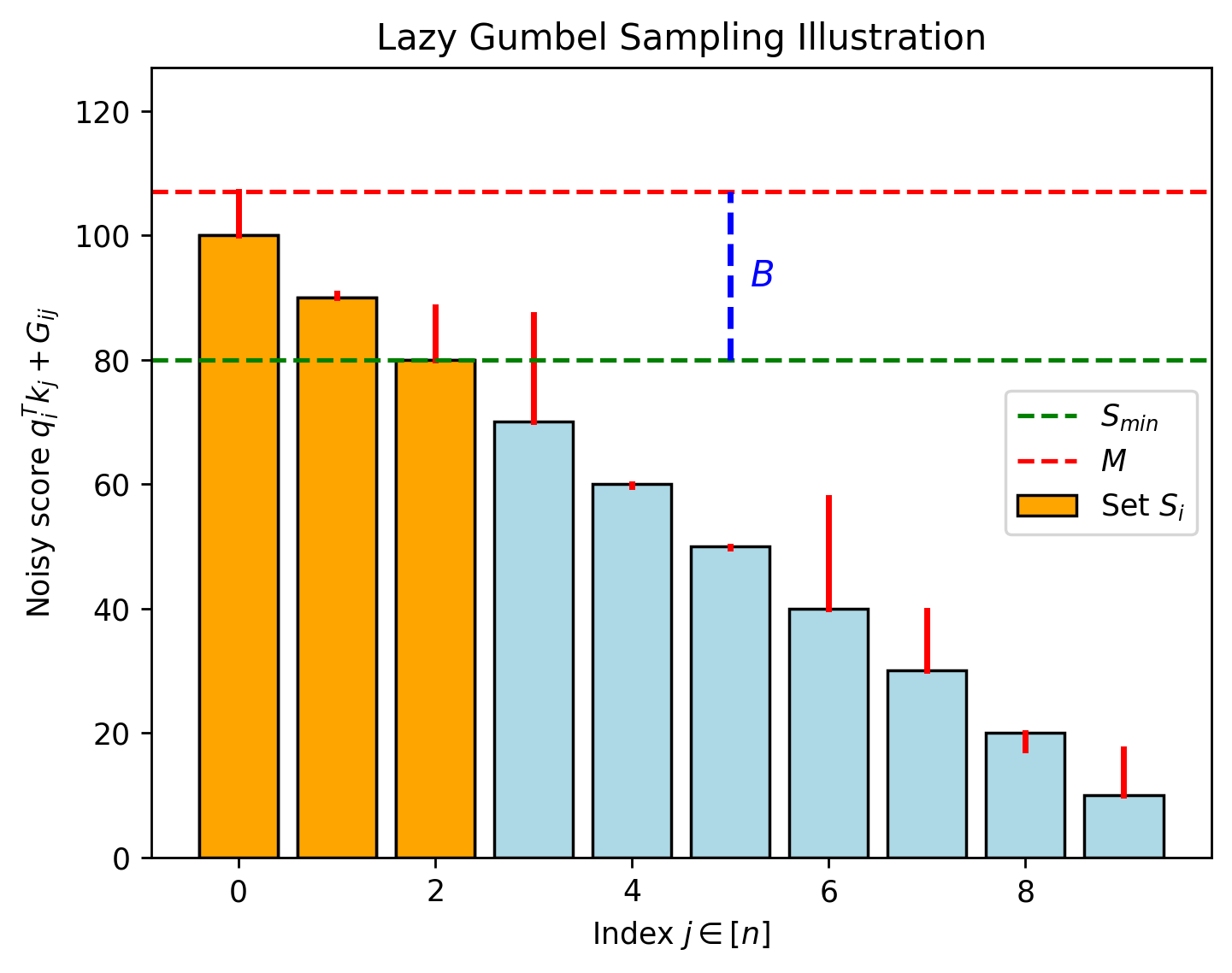}
    \vspace{-15pt}
    \caption{Lazy Gumbel sampling}
    \label{fig:lazy-gumbel}
    \vspace{-10pt}
\end{wrapfigure}

In Appendix \ref{appendix:m_bound_proof}, we show that the expected number $m$ of large Gumbels is at most $n/k$. 
Our simplified proof uses the Gumbel distribution's Moment Generating Function, 
rather than the original exponential-based analysis.

% Then, the run-time of this algorithm is $O\left(k + \frac{n}{k}\right)$, 
% implying we should set $k = \sqrt{n}$ to balance out the terms. 
% Suppose $f(n,\sqrt{n})$ is the time it takes us to actually obtain the top 
% $\sqrt{n}$ scores. 
%Then the total runtime is $O(\sqrt{n} + f(n,\sqrt{n}))$. 
%In our applications, we'll aim for $f(n,\sqrt{n}) = O(\sqrt{n})$, which will give 
%us a total runtime of $O(\sqrt{n})$.

\begin{lemma}
\label{lemma:large-gumbel-bound}
The following holds:
\begin{align}
    \mathbb{E}\left[m\right] \leq \frac{n}{k}
\end{align}
\end{lemma}

Due to Lemma \ref{lemma:large-gumbel-bound}, we see that we need to set 
$k = \sqrt{n}$ to optimize our overall time complexity. 
We have the following theorem, which follows easily from the pseudocode of 
Algorithm \ref{alg:lazy-gumbel} and Lemma \ref{lemma:large-gumbel-bound}:
\begin{theorem}
\label{thm:runtime-analysis-lazy-gumbel}
Let $k = \sqrt{n}$. Suppose that we are able to retrieve the set $S_i$ in 
$f(n,k)$ time. 
Then, we can use Algorithm \ref{alg:lazy-gumbel} to sample from $D_i$ in 
$O(\sqrt{n} +  f(n,\sqrt{n}))$ time in expectation. 
\end{theorem}

\subsubsection{Obtaining the top $k$ inner products}
\label{sec:top_k_analysis}
Algorithm \ref{alg:lazy-gumbel} relies on obtaining the set \(S_i\) of the top \(\sqrt{n}\) inner products \(q_i^T k_j\) for each \(i \in [n]\) in sub-quadratic $f(n,\sqrt{n})$ time. Since the \(k_j\) vectors are fixed, while the \(q_i\) vectors act as queries, this setup is known as the\textbf{ \(k\)-Maximum Inner Product Search Problem (MIPS)}

The \(k\)-MIPS problem can be reduced to the $k$NN problem using a transformation proposed by \cite{neyshabur2015symmetric}. We add an extra dimension to normalize all key vectors. Specifically, the inner product \(q_i^T k_j\) can be expressed as:
\begin{align}
q_i^T k_j = \frac{1}{2}\left(||q_i||_2^2 + ||k_j||_2^2 - ||q_i - k_j||_2^2\right)
\end{align}
If the norms \(||k_j||_2\) are the same across all \(j\), the problem reduces to finding the \(k\) nearest neighbors to \(q_i\). To enforce this, we define:
\begin{align}
\left(k'_j\right)^T = \left[k_j^T, \sqrt{M-||k_j||_2^2}\right]
\end{align}
so that \(||(k')_j||_2 = M\) for all \(j \in [n]\), where $M$ is a previously known upper bound. When querying with \(q_i\), we use:
\begin{align}
(q'_i)^T = \left[q_i^T, 0\right]
\end{align}
This transformation preserves the inner products, allowing us to solve the $k$NN problem for \(q'_i\). We can then use a $k$NN index \(H\) to preprocess \(K\) and query it with each \(q'_i\) to construct \(S_i\) for all \(i \in [n]\). We remain agnostic to the specific $k$NN index one could use for this algorithm\footnote{For a specific construction with precise theoretical guarantees that uses Locality Sensitive Hashing (LSH), please refer to Appendix \ref{section:lsh_attn}.}, but if we assume that the construction runtime is slightly larger than linear and the query time slightly larger than $k$, then $k$NN attention techniques have total runtime of $\approx \widetilde{O}(dn^{3/2}\cdot \varepsilon^{-2}\log(1/\delta))$ time and space.

\begin{algorithm}
\caption{$k$NN Attention}\label{alg:knn_attn}
\begin{algorithmic}[1]
\State \textbf{Inputs: }$Q,K,V \in \mathbb{R}^{n\times d}$, error parameter $\varepsilon > 0$, confidence parameter $\delta > 0$, $k \in \mathbb{N}$.
\vspace{1.5mm}
\For{$j \in [n]$}\Comment{Pre-Processing}
    \State $(k'_j)^T = \left[k_j^T, \sqrt{M - ||k_j||_2^2}\right] \in \mathbb{R}^{(d+1)\times 1}$
\EndFor
\State $H \gets$ Build a $k$NN index from $\{k'_j\mid j\in [n]\}$
\vspace{1mm}
\For{$i \in [n]$}
    \vspace{0.75mm}
    \State $(q'_i)^T \gets [q_i^T, 0] \in \mathbb{R}^{d+1}$
    \vspace{0.5mm}
    \State Query $H$ with $q'_i$ to get $S_i$ with $|S_i| = k$.
    \vspace{0.5mm}
    \For{$j \in [d]$}
        \State $\widehat{O}_{ij} \gets$ Median-Of-Means with Algorithm \ref{alg:lazy-gumbel} as sampler $\gets (k,q_i,K,S_i)$.
    \EndFor
\EndFor
\State \Return $\widehat{O}$
\end{algorithmic}
\end{algorithm}
\subsection{$k$NN Attention without Median-of-Means}
This section describes a simpler algorithm for computing the expected value needed for self-attention. The algorithm still uses $k$NN indices to find the top $k$ inner products per query, but usually outperforms Algorithm \ref{alg:knn_attn} in practice, due to its amenity for hardware-accelerated vectorization, and is thus our preferred implementation for experiments\footnote{See Appendix \ref{section:vectorized_alg_appendix} for a PyTorch implementation of this algorithm.}.

Building on \cite{mussmann2017fast}, the algorithm estimates \(\mathbb{E}_{k\sim D_i}[V_{kj}]\) using set \(S_i\) by sampling \(\ell\) additional vectors outside \(S_i\) (set \(T_i\)) and upweighting them in the expectation sum, as follows:
\begin{align}
\widehat{O}_{ij} = \frac{\sum_{s\in S_i} e^{q_i^T k_s} \cdot V_{sj}+\frac{n-k}{\ell}\sum_{s\in T_i} e^{q_i^T k_s} \cdot V_{sj}}{\sum_{s\in S_i}e^{q_i^T k_s} + \frac{n-k}{\ell}\sum_{s\in T_i} e^{q_i^T k_s}}
\end{align}
The quality of this estimator and the optimal choices for \(k\) and \(\ell\) are derived as follows:

\begin{theorem}
\label{thm:simpler_expectation}
The estimator $\widehat{O}_{ij}$ satisfies the following error guarantee with probability at least $1-\delta$:
\begin{align*}
    \left|\widehat{O}_{ij} - O_{ij}\right| = O(\varepsilon)
\end{align*}
if the following two conditions hold:
$k^2 \ell \geq 8n^2\varepsilon^{-2}\log\left(4/\delta\right)$ and $
k\ell \geq 2n\varepsilon^{-2}\log\left(2/\delta\right)$. Setting $k = \ell = O\left(n^{2/3}\varepsilon^{-1}\sqrt{\log(1/\delta)}\right)$
gives us an $\widetilde{O}\left(dn^{5/3} \varepsilon^{-1}\sqrt{\log(1/\delta)}\right)$ algorithm for estimating self-attention within additive error $O(\varepsilon)$, assuming an efficient $k$NN implementation.
\end{theorem}

\begin{proof}
The proof of the additive error guarantee can be found in \cite{mussmann2017fast}. 
\end{proof}

\section{Approximating the Attention Gradients}
Next, we present randomized algorithms which can efficiently approximate the gradients of the self-attention function. First, we give exact formulas for the gradients in question. These can be obtained by applying the chain rule repeatedly, as shown in Appendix \ref{sec:grad-derivation}.

\begin{lemma}[Attention Gradients]
Let $Q,K,V \in \mathbb{R}^{n\times d}$. Let $P := \text{softmax }(Q K^T) \in \mathbb{R}^{n\times n}$  be the normalized attention matrix. Let $\phi$ be a scalar function of $O$ and $D^O = \partial \phi / \partial O \in \mathbb{R}^{n\times d}$. Similarly define $D^Q, D^K$ and $D^V$. The following relationships hold:
\begin{align}
    D^V &= P^T \cdot D^O\\
    D^Q_{ij} &= \sum\limits_{k=1}^n P_{ik}\left(D^P_{ik} -\langle D^P_{i,:}, P_{i,:}\rangle\right)K_{kj}\\
    D^K_{ij} &= \sum\limits_{k=1}^n P_{ki}\left(D^P_{ki} - \langle D^P_{k,:}, P_{k,:}\rangle\right)Q_{kj}
\end{align}
where $D^P_{ij} = \partial \phi / \partial P_{ij} = \langle D^O_{i,:}, V_{j,:}\rangle$.
\end{lemma}

Clearly, calculating $D^Q, D^K$ and $D^V$ naively requires storing $P$, which requires $O(dn^2)$ time.

\subsection{Estimating $D^V$ using Random Walk Simulations}
We now give an algorithm for estimating $D^V$. Suppose we want to calculate the $j$-th column of $D^V$: 
\begin{align}
D^V_{:,j} = P^T\cdot D^O_{:,j}
\end{align}
for $j \in [d]$. Fix $\vv{x_j} := D^O_{:,j} \in \mathbb{R}^{n \times 1}$ and suppose that $\vv{x_j} \geq 0$. We will relax this assumption in Section \ref{sec:negative_numbers}. Then, $\vv{y_j} := \vv{x_j} / ||\vv{x_j}||_1$ is a distribution over the universe $[n]$. Imagine a random walk over $[n]$ with transition matrix $P$ and initial distribution $\vv{y_j}$. Then:
\begin{align}
\vv{\pi_j} := P^T \cdot \vv{y_j}
\end{align}
is the distribution after one step in the process. Thus, we can estimate $\vv{\pi_j}$ with Markov Chain simulations, by first picking an item $i \in [n]$ from the distribution $\vv{y_j}$, and then picking another item $k \in [n]$ with probability $P_{ik}$. We make $N$ independent length-$1$ random walks like this and let:
\begin{align*}
X_{v}^{(j,s)} = 
\begin{cases}
    1,&\text{ if the $s$-th walk ends up in state $v$}\\
    0,&\text{ otherwise}
\end{cases}
\end{align*}
We know that $\mathbb{E}[X_{v}^{(j,s)}] = \pi_j(v)$ for all $s \in [N]$. Thus, we can form a boosted estimator:
\begin{align}
    \widehat{p_{j}}(v) = \frac{1}{N}\sum\limits_{s=1}^N X_{v}^{(j,s)}
\end{align}

\begin{figure}[ht]
    \centering
    \includegraphics[width=0.5\textwidth]{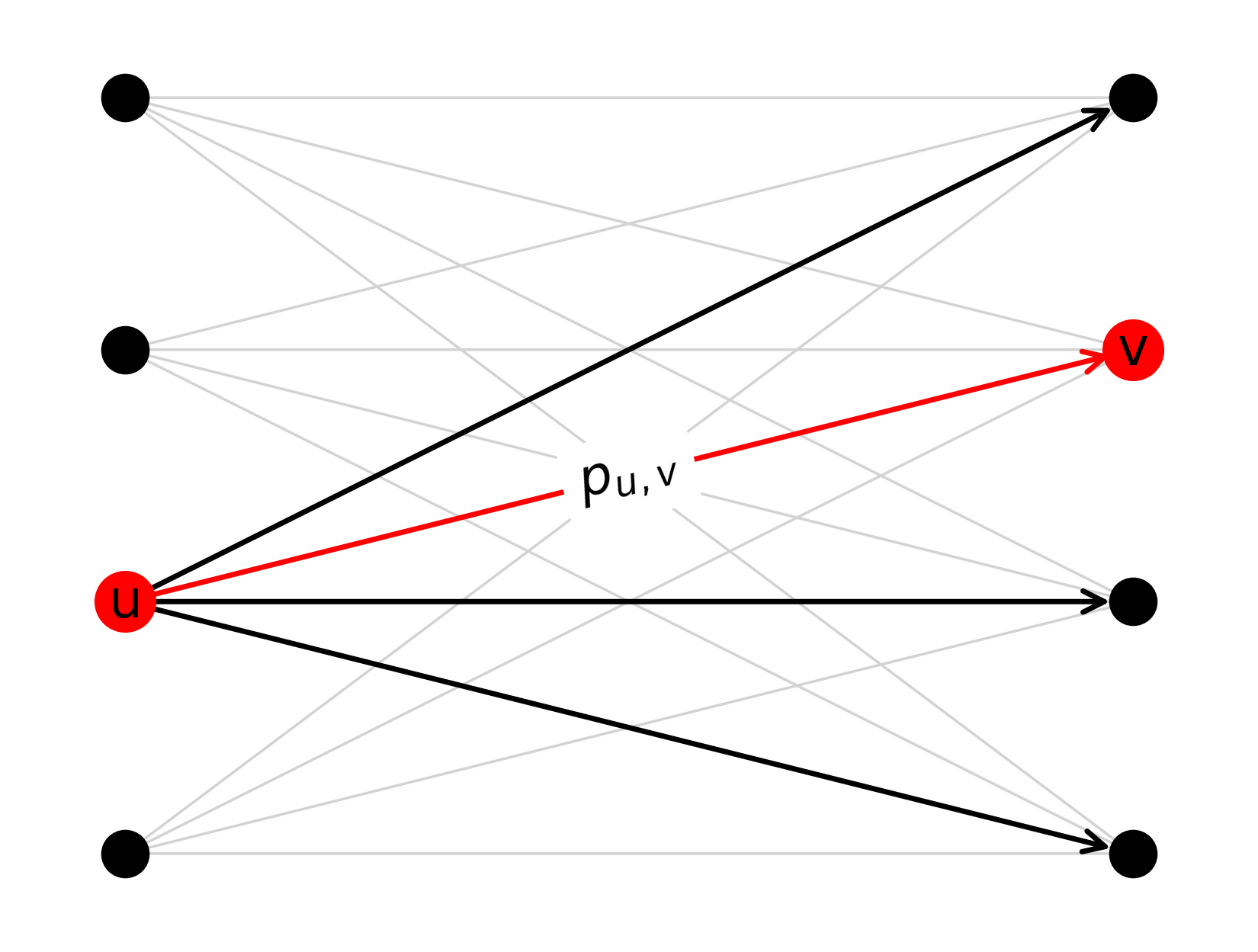}
    \caption{A single-step Markov Chain sample. }
\end{figure}
    
This estimator is unbiased due to linearity of expectation, so we can use the Hoeffding bound to ensure that our empirical distribution is close to the true distribution as long as we take enough samples:
\begin{align*}
\Pr\left[|\widehat{p_j}(v)-\pi_j(v)| \geq \varepsilon\right] \leq 2\exp(-2N\varepsilon^{2})
\end{align*}
We set the probability of failure to $1/(dn^2)$ so that we can union bound over all $v \in [n]$ and all $j \in [d]$. It follows that we require:
\begin{align}
N = \Theta(\varepsilon^{-2}\ln (nd))
\end{align}
Of course, we need to scale $\vv{\pi_j}$ back to recover $D^V_{:,j}$. We define:
\begin{align}
    \widehat{D}^V_{:,j} = ||\vv{x_j}||_1 \cdot \widehat{p_j}
\end{align}
Then we get that with probability at least $1-1/n$ it holds for all $j \in [d]$ that:
\begin{align}
\left|\left|\widehat{D}^V_{:,j} - D^V_{:,j}\right|\right|_\infty = ||\vv{x_j}||_1 \cdot \left|\left|\widehat{p_j}-\vv{\pi_j}\right|\right|_\infty \leq \varepsilon ||\vv{x_j}||_1
\end{align}

\subsubsection{Relaxing the non-negativity assumption}
\label{sec:negative_numbers}
We now relax the non-negativity constraint on $\vv{x_j}$, accepting some approximation error. Since normalizing $\vv{x_j}$ with its L1 norm fails if $\vv{x_j}$ has negative entries, we adopt a numerical stability technique to ensure we get a valid probability distribution even when $\vv{x_j}$ contains negative entries. Let 
\begin{align}
M_j := -\min\limits_{\substack{v\in[n] \\ (\vv{x_j})_v \leq 0}}(\vv{x_j})_v
\end{align}
be the absolute value of the most negative entry of $\vv{x_j}$. If $\vv{x_j} \geq 0$, then set $M_j = 0$. Now, if $\vv{M_j} := M_j \cdot 1^{n} \in \mathbb{R}^{n\times 1}$, then $\vv{x_j'} = \vv{x_j} + \vv{M_j} \geq 0$. Therefore, we can estimate
\begin{align}
    \widehat{p'}_j \approx \vv{\pi_j'} := P^T \cdot \vv{x_j'}
\end{align}
using our Markov Chain method. Going back to our original goal of estimating $\vv{\pi}_j$, we have:
\begin{align*}
    \vv{\pi_j} &:= P^T \cdot \vv{x_j} \\
    &= P^T \cdot (\vv{x_j'} - \vv{M_j}) \\
    &= \vv{\pi_j'} - P^T \cdot \vv{M_j} \\
    &= \vv{\pi_j'} - M_j\cdot P^T \cdot 1^n
\end{align*}
where $1^n$ is the all $1$-s vector. So then we only need to additionally estimate $P^T \cdot 1^n$. This can be done in the same fashion only once, as a pre-processing step. Specifically, suppose that we estimate $P^T \cdot 1^n$ as $\widehat{s}$. We know from our prior analysis that using $\Theta(\varepsilon^{-2}\log n)$ random walks we get an estimate $\widehat{s}$ such that:
\begin{align}
\left|\left|\widehat{s}-P^T \cdot 1^n\right|\right|_\infty \leq \varepsilon n
\end{align}
Putting it all together, our final estimator is then:
\begin{align}
    \widehat{p}_j := \widehat{p'}_j - M_j\cdot \widehat{s}
\end{align}
Eventually, the total error for estimating $\vv{\pi_j}$ becomes:
\begin{align}
    \left|\left|\widehat{D}^V_{:,j} - D^V_{:,j}\right|\right|_\infty = \left|\left|\widehat{p}_j-\vv{\pi_j}\right|\right|_\infty
    &= \left|\left|\widehat{p'}_j - M_j\cdot \widehat{s} - \vv{\pi'_j} + M_j \cdot P^T \cdot 1^n\right|\right|_\infty\\
    &\leq \left|\left|\widehat{p'}_j-\vv{\pi'_j}\right|\right|_\infty + M_j\cdot \left|\left|P^T 1^n-\widehat{s}\right|\right|_\infty \\
    &\leq \varepsilon\left|\left|\vv{x'_j}\right|\right|_1 + \varepsilon M_j\cdot n \\
    &= \varepsilon\langle x_j,1^n\rangle + 2\varepsilon nM_j
\end{align}
where the first inequality follows from the triangle inequality and the last equality follows from 
\begin{align}
\left|\left|\vv{x'_j}\right|\right|_1 = \sum\limits_{k=1}^n \left|(x_j)_k+M_j\right| = \sum\limits_{k=1}^n \left[(x_j)_k + M_j\right] = \langle x_j, 1^n\rangle + nM_j
\end{align}

We present the entirety of the method as Algorithm \ref{alg:dv_estimation}.

\begin{algorithm}
\caption{Estimating $P^T x$, with query access to $P \in \mathbb{R}^{n\times n}$ a stochastic matrix}
\label{alg:mcmc_alg}
\begin{algorithmic}[1]
\Procedure{ApproxPosProd}{$P \in \mathbb{R}^{n\times n},x \geq 0,\varepsilon > 0$}
    \State Let $N \gets 2\lg n\cdot \varepsilon^{-2}$ and $\Sigma \gets \langle x, 1^n \rangle$ \Comment{$O(n)$ time}
    \State Let $\widehat{x} \in \mathbb{R}^{n\times 1}$ be our output.
    \For{$s \in [N]$}
        \State Sample $i \in [n]$ with probability $\propto x_i$ using $\Sigma$ as a normalization factor.
        \State Sample $k \in [n]$ with probability $P_{ik}$. 
        \State $\widehat{x}_{k} \gets \widehat{x}_{k} + 1$
    \EndFor
    \State \Return $\frac{1}{N}\cdot \widehat{x}\cdot \Sigma$
\EndProcedure
\Procedure{EstimateProduct}{$P \in \mathbb{R}^{n\times n},x \in \mathbb{R}^n,\varepsilon > 0,\widehat{s} \in \mathbb{R}^{n}$}
\State Let $M \gets -\min_{v \in [n],x_v \leq 0}x_v$ \Comment{$O(n)$ time}
\State Let $x' \gets x + M\cdot 1^n$ \Comment{$O(n)$ time}
\State Call \Call{ApproxPosProd}{$P,x',\varepsilon$} to get $\widehat{x'}$ \Comment{$\widetilde{O}(n\varepsilon^{-2})$ time}
\State \Return $\widehat{x'} - M\cdot \widehat{s}$
\EndProcedure
\end{algorithmic}
\end{algorithm}

\begin{algorithm}
\caption{Estimating $D^v
V$}\label{alg:dv_estimation}
\begin{algorithmic}[1]
\State \textbf{Input: }$Q,K,D^O \in \mathbb{R}^{n\times d}$, error parameter $\varepsilon > 0$
\State Let $\widehat{D}^V \in \mathbb{R}^{n\times d}$ be our output.
\State $\widehat{s} \gets $\Call{ApproxPosProd}{$P, 1^n, \varepsilon$}\Comment{Pre-Processing}
\For{$j \in [d]$}
    \State $\widehat{D}^V_{:,j} \gets $\Call{EstimateProduct}{$P, D^O_{:,j}, \varepsilon, \widehat{s}$}
\EndFor
\State \Return $\widehat{D}^V$
\end{algorithmic}
\end{algorithm}

\subsubsection{Runtime analysis}
For each $j \in [d]$ we take $N = O(\varepsilon^{-2}\log n)$ samples. We can take one sample in $O(nd)$ time. In addition, we must pre-calculate the sums $\langle x_j, 1^n \rangle + nM_j = nM_j + \sum_{k=1}^n D^O_{kj}$ for all $j\in [d]$, which takes $O(nd)$ time. As a result, we arrive at the following theorem:
\begin{theorem}
Given $Q,K,V$ and $D^O$, Algorithm \ref{alg:dv_estimation} calculates ${\partial \phi} /{\partial V_{ij}} = D^V_{ij}$ for all $(i,j) \in [n]\times[d]$ within an additive approximation error of
\begin{align}
    e_V = \varepsilon\cdot \langle D^O_{:,j}, 1^n\rangle + 2n\varepsilon M_j,\,\text{where }M_j := -\min_{i \in [n], D^O_{ij} \leq 0}D^O_{ij}
\end{align}
with probability at least $1-\frac{1}{n}$. The time complexity is $O(nd^2\varepsilon^{-2}\log n)$.
\end{theorem}

\begin{remark}
Note that Algorithm \ref{alg:dv_estimation} does not materialize the $P$ matrix. Instead it accesses its elements by using $Q$ and $K$ in $O(d)$ time per element.
\end{remark}

\section{Experimental Results}
In this section we present our experimental results. Through them we can interpret our theoretical framework better and solidify our understanding of it.

\subsection{Forward Pass Approximation Quality on Random Inputs}
We begin by evaluating the effectiveness of $k$NN Attention in approximating the attention function. We randomly sample matrices \(Q, K, V \in \mathbb{R}^{n \times d}\) from a uniform distribution over \([-B, B]^{n \times d}\) and assess the approximation quality on these inputs. Our focus is on the ``classic'' $k$NN Attention estimator (Theorem \ref{thm:simpler_expectation}) with \(\lambda = 1\), as used in implementations like \cite{bertsch2024unlimiformer} and \cite{wu2022memorizing}. We vary \(k\) to study how the error decreases as \(k\) increases and compare the efficiency to the naive \(O(n^2)\) attention, expecting notable performance gains. This experiment is implemented in PyTorch, running on a MacBook Air with an M3 CPU and 8GB of RAM.

\paragraph{Efficiency of kNN Attention}
Our experiments confirm $k$NN Attention's superior speed, demonstrating sub-quadratic scaling. With a batch size of $1$ and $H = 10$ attention heads, it handles self-attention for $n=10^6$, while the naive method runs out of memory beyond $n \geq 20000$. Increasing $k$ further leads to memory errors for $n \geq 50000$, highlighting kNN Attention's memory efficiency. Detailed results are in Figure \ref{fig:efficiency}. 

\begin{figure}
\centering
\subfigure[$k$NN Attention vs Brute Force. We are able to increase the context length by a factor of $5$ without running out of memory.]{
\label{fig:efficiency}
\includegraphics[scale=0.42]{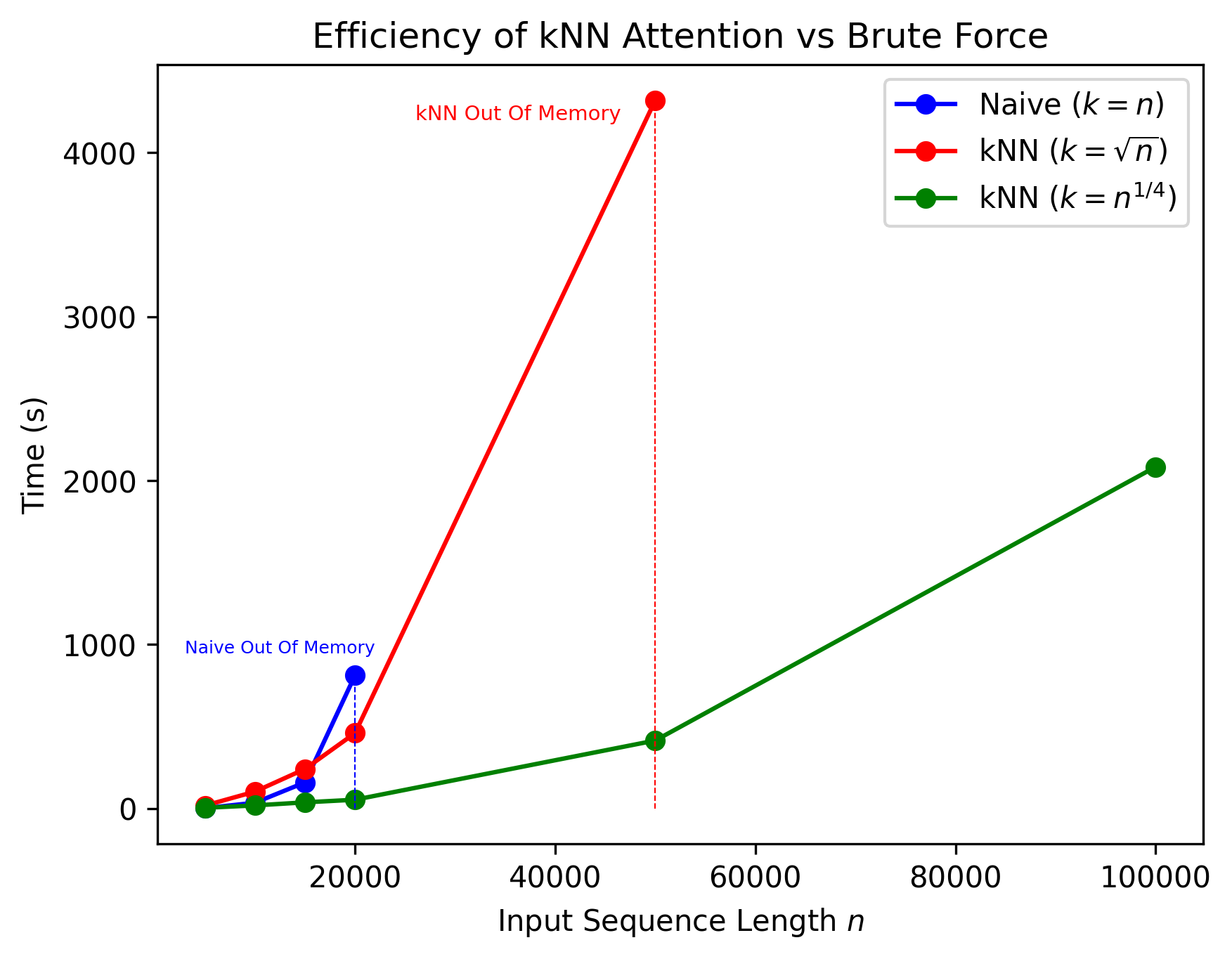}
}
\hspace{1mm}
\subfigure[Mean error of $k$NN Attention as a function of $k$ for different values of $B$. As $k$ grows, the error becomes negligible. In some cases, $\sqrt{n}$ is too big a threshold.]{
\label{fig:error}
\includegraphics[scale=0.43]{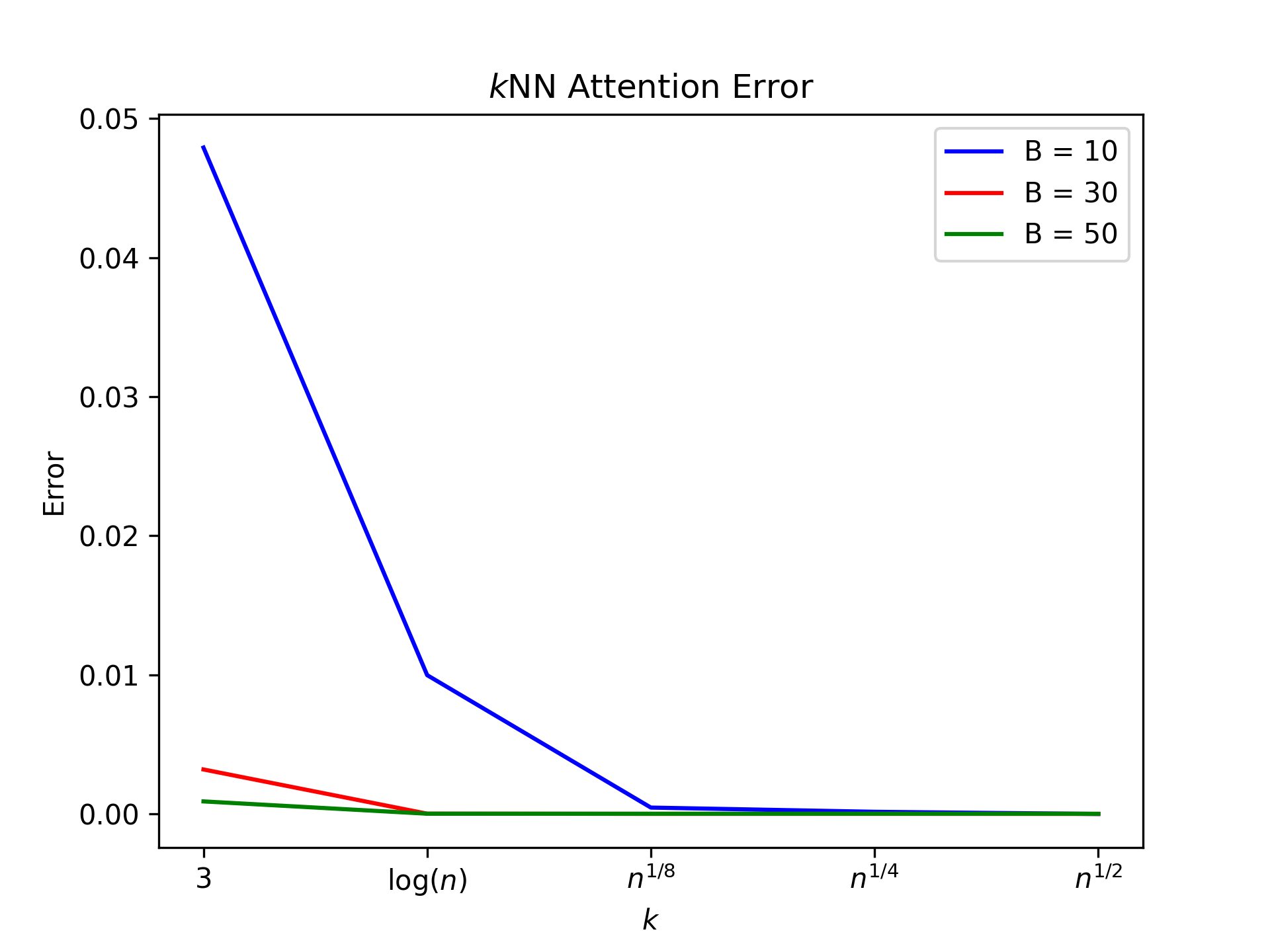}
}
\label{fig:errors}
\end{figure}

\paragraph{Role of $k$ in the Approximation Error}
We investigate the impact of \(k\) on the approximation error, predicting that error increases as \(k\) decreases. The experiment confirms this, showing that for \(k \geq n^{1/8}\), the error is minimal. Our theory suggests a threshold closer to \(\sqrt{n}\) which indicates that the optimal \(k\) may vary by dataset. The results appear in Figure \ref{fig:error}. Interestingly, the error is more pronounced for small values of both $B$ and $k$, potentially due to the limited approximation power when $k$ is small. For larger $k$, this difference becomes negligible.

\subsection{Backward Pass Approximation Quality}
\label{sec:grad-approx-experiments}
Next, we evaluate the quality of our algorithms for attention gradient estimation. We sample \(Q, K, V\) from a normal distribution, as this strategy aligns with typical neural network weight initialization strategies, and approximate $D^Q$ and $D^V$ using randomized techniques. Our goal is to assess the error introduced by the approximation and whether this error causes gradient descent to converge far from the minimum.

We set the sequence length \(N = 100\) and the embedding dimension \(d = 3\). The learning rate \(\alpha\) is varied between 0.05 and 0.5, while the error parameter is fixed at \(\varepsilon = 0.05\) and the confidence parameter at \(\delta = 0.1\). We experiment with both convex (Mean Square Error) and non-convex (Cross Entropy) loss functions to examine how approximate gradient descent behaves, using PyTorch's autograd to compute the exact attention gradients. As shown in Figure \ref{fig:grad-desc-experiment}, our approximation closely matches the expected results in the convex case but deviates from the optimal convergence in the non-convex case. A more detailed investigation of the impact of gradient approximations in large language model (LLM) training is left for future work.

\begin{figure}[h!]
\centering
\subfigure{
\includegraphics[scale=0.45]{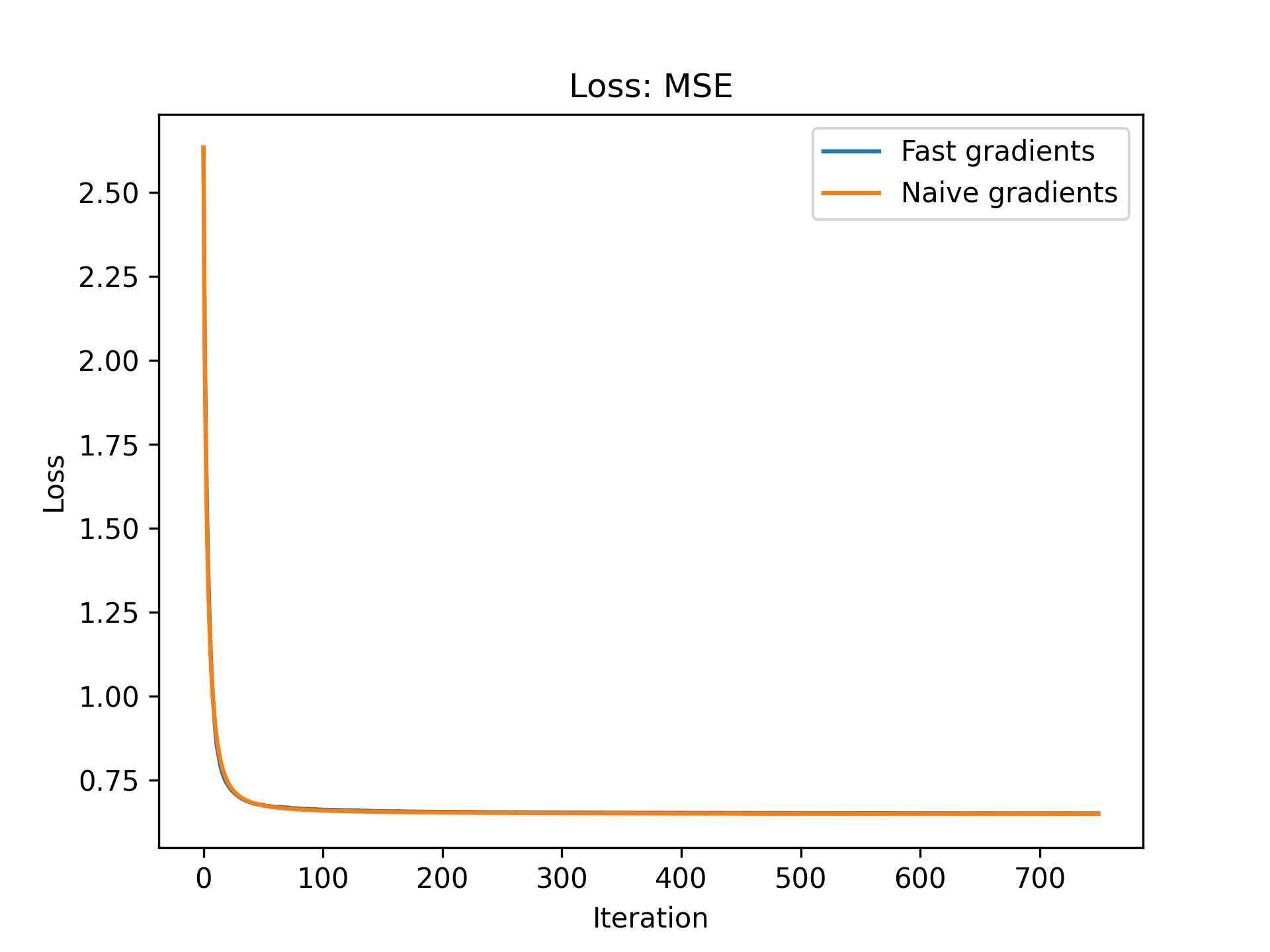}
}
\subfigure{
\includegraphics[scale=0.45]{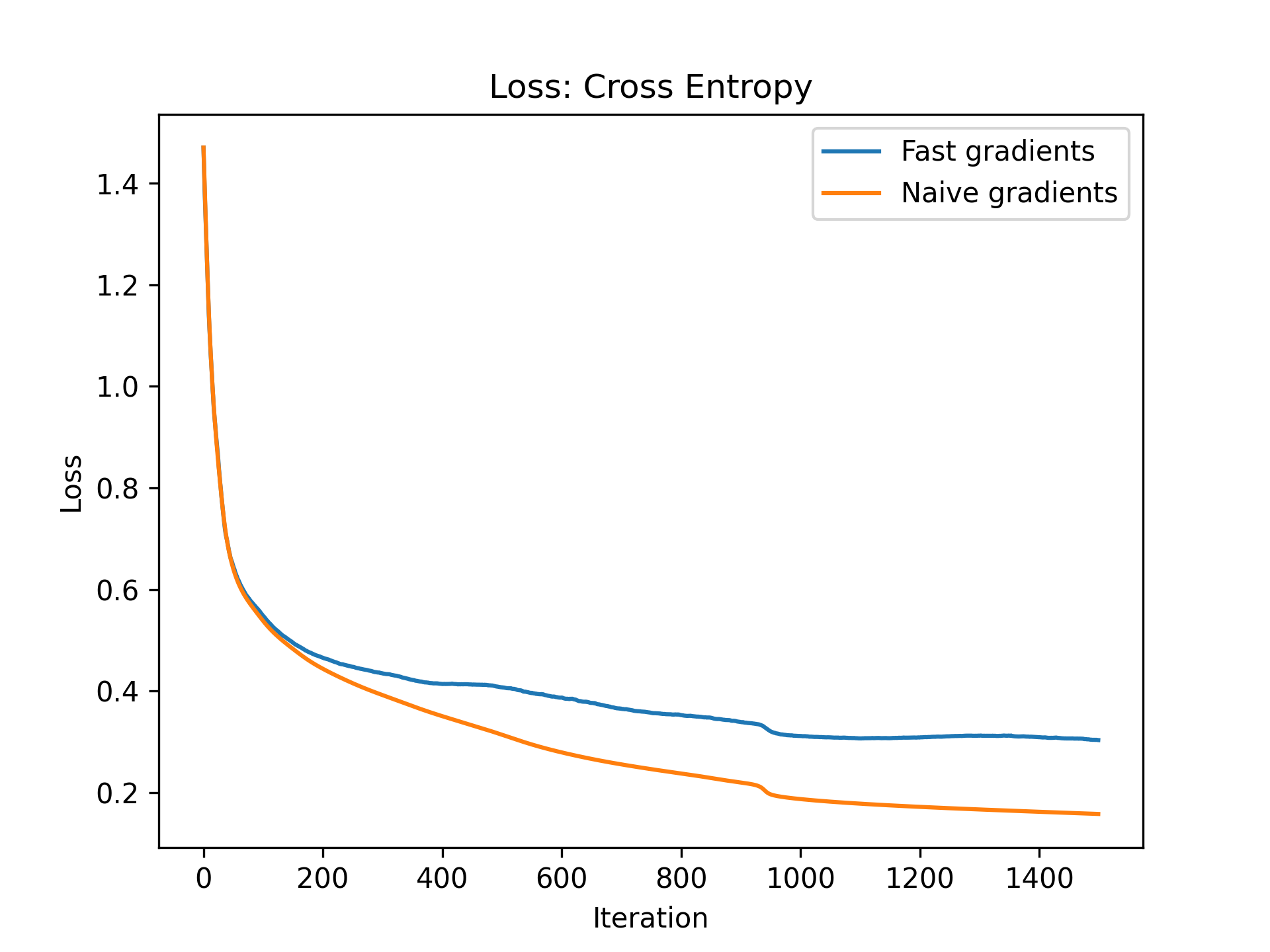}
}
\caption{Gradient Descent with Approximate Gradients against different loss functions $\phi$. Even with approximate gradients, gradient descent still makes adequate progress towards convergence.}
\label{fig:grad-desc-experiment}
\end{figure}

\subsection{Experiments on LLMs}
Finally, we experiment with incorporating $k$NN attention into LLMs to study its impact on training and inference. While previous work has explored $k$NN indices for efficient fine-tuning and training \citep{bertsch2024unlimiformer, wu2022memorizing}, our goal is to understand how Transformer LLMs respond to attention function approximation, linking it to our theory and providing practical guidelines. The architecture and training methods are adapted from \textit{nanoGPT} \citep{Karpathy2022}, and our experiments are conducted on an NVIDIA L40 GPU with 48GB of memory.

Our first experiment trains a mini character-level Transformer on a small Shakespeare dataset, replacing attention with $k$NN attention. We compare training and validation perplexity to the exact method. Results in Figure \ref{fig:shakespeare_experiment} show that $k$NN Attention maintains a small perplexity gap. However, as overfitting occurs, the perplexity difference widens, possibly due to increasing maximum approximation error. Future work could explore the impact of larger $k$ values on perplexity.

We also experiment with fine-tuning a large pre-trained LLM using $k$NN attention, in the hopes that the approximation will not severly degrade the model's quality. We manage to fine-tune GPT-2 XL\footnote{$1.61$B parameters, see \cite{radford2019language}} on a Shakespeare dataset - a task typically infeasible on a single L40 GPU due to memory constraints with quadratic attention. Examples from prompting this model can be found in Appendix \ref{appendix:samples}.

\begin{figure}[h]
\center
\subfigure{
  \centering
  \includegraphics[scale=0.46]{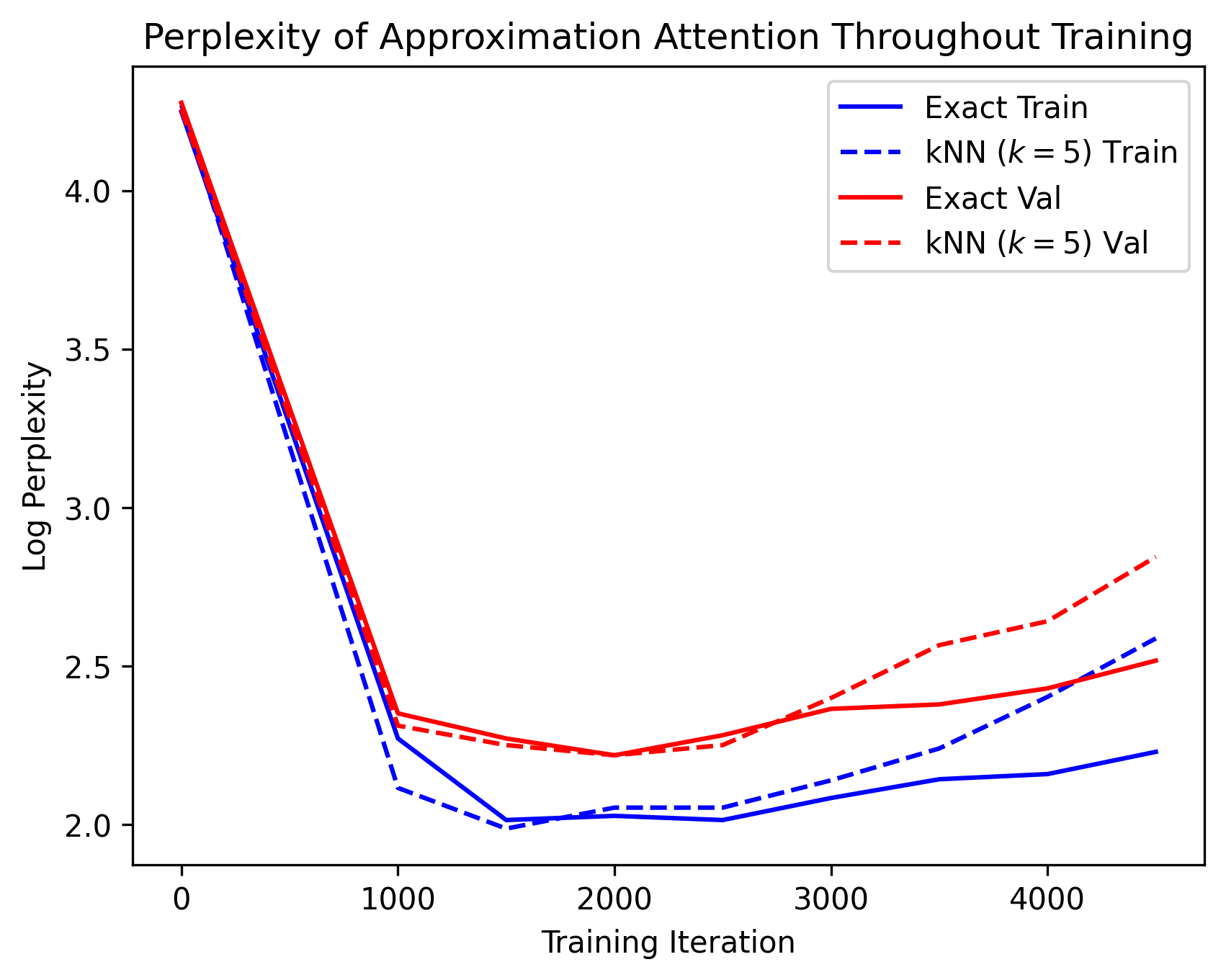}
}
\subfigure{
  \centering
  \includegraphics[scale=0.46]{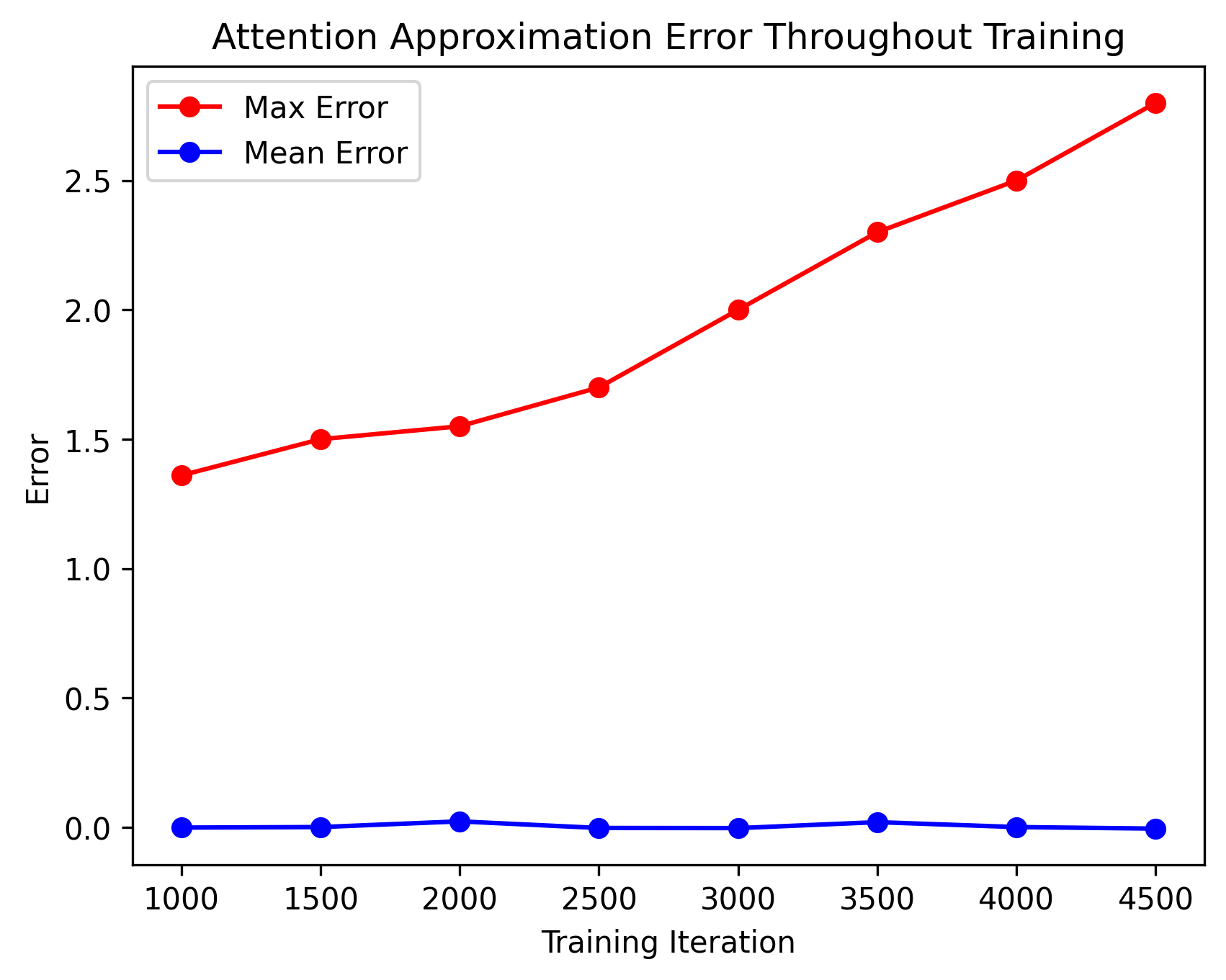}
}
\caption{The perplexity and approximation error of $k$NN Attention throughout training}
\label{fig:shakespeare_experiment}
\vspace{-3mm}
\end{figure}
\section{Conclusion}
In this work, we developed a theoretical framework for leveraging $k$NN techniques to design 
efficient and effective Transformer architectures.
We extended this framework by introducing Markov Chain-based methods to propose novel algorithms 
for efficient self-attention gradient computation. Empirical validation on both synthetic inputs and real-world datasets demonstrated that $k$NN approximations 
closely match the original performance in LLM training, while significantly reducing
computational costs during both training and inference. 

Moreover, our research opens several avenues for future exploration. 
Key questions include the effectiveness of training with approximate gradients compared to exact ones,
particularly when error distributions are tightly concentrated but unknown. Another open question is about explaining the practical 
observation that the optimal $k$ value is often significantly smaller than the predicted $\sqrt{n}$.

In conclusion, by applying sublinear algorithm techniques, our work provides a solid foundation 
for making Transformers more scalable and efficient while identifying critical areas for further 
research in LLM approximation.

\section*{Reproducibility Statement}
To aid in the reproduction of our experiments, we include our code in the repository at \url{https://github.com/sansui-123/knn_attention}. For our theoretical contributions, full proofs of our claims and analyses of our algorithms can be found in the Appendix.

\section*{Acknowledgements}
We thank Andrea Lincoln, Thien Nguyen, and Krzysztof Onak for valuable discussions at the early stages of this project. We are also grateful to Brian Kulis and Esty Kelman for their invaluable feedback on the manuscript and to Nikos Georgoudios for his assistance with the MGF proof of Lemma \ref{lemma:large-gumbel-bound}.

\bibliography{bib}
\bibliographystyle{plainnat}

\appendix
\section*{Appendix}
In the following sections we deposit theoretical results, proofs and algorithms that are missing from the main paper, either due to space constraints or for the sake of clarity.

\section{Preliminaries}
\label{section:appendix_prelims}

\subsection{Self-Attention and Approximation}
We start by defining self-attention and some of its variants.
\begin{definition}[Self-Attention]
Let $Q,K,V \in \mathbb{R}^{n \times d}$. We can think of these matrices as a collection of $n$ $d$-dimensional \textit{query, key} and \textit{value} vectors respectively. We define the \textbf{self-attention} function as follows:
\begin{align}
    O(Q,K,V) = D^{-1}AV
\end{align}
where $A = \exp(QK^T) \in \mathbb{R}^{n\times n}$ is the \textbf{attention matrix} and $D = \text{diag}(A1^n)$. $D$ effectively implements taking a row-wise softmax of the entries of $QK^T$. Many modern implementations of attention consider \textbf{causal attention}, in which we mask away the upper-triangular entries of $A$. 
\end{definition}

\begin{remark}[Normalization by $\sqrt{d}$]
In most implementations we divide $QK^T$ by $\sqrt{d}$ \citep{vaswani2017attention} because it reduces the variance of each element of $A$ had $Q,K,T$ been selected from a uniform distribution. We will omit this technicality because it does not affect our algorithmic techniques. For the rest of this paper we will assume that $d^{-1/2}$ has been pre-normalized into $K$.
\end{remark}

\begin{remark}[Dropout]
Many attention implementations also use \textbf{dropout}. Every entry of $A$ will be masked to $0$ with probability $p$, where $p$ is set to a small constant, like $0.1$. 
\end{remark}

\begin{definition}[$(\varepsilon,\delta)$-estimators]
Let $X$ be a statistic and $\widehat{X}$ be an estimator we have for it. $\widehat{X}$ is an $(\varepsilon,\delta)-$\textbf{additive estimator} if with probability at least $1-\delta$ it holds that
$$
\left|X-\widehat{X}\right| \leq \varepsilon
$$
Respectively, we call the estimator \textbf{multiplicative} if
$$
\left|X-\widehat{X}\right| \leq \varepsilon X
$$
\end{definition}
We will often make use of the following boosting lemma from the theory of randomized algorithms:

\begin{lemma}[Median-Of-Means Amplification Technique, \citep{chakrabarti2020data}]
\label{lemma:median-of-means_appendix}
If $\widehat{Q}$ is an unbiased estimator of some statistic,
then one can obtain an $(\varepsilon,\delta)$-multiplicative estimate of that statistic by suitably combining

$$
K := \frac{C}{\varepsilon^2}\ln \frac{2}{\delta}\frac{\text{Var}[\widehat{Q}]}{\mathbb{E}[\widehat{Q}]^2}
$$
independent samples of $\widehat{Q}$, where $C$ is a universal constant. 
\end{lemma}

\subsection{Gumbel Noise}
The \textbf{Gumbel Distribution} will be useful for sampling from softmax distributions. We define it below:
\begin{definition}[Gumbel Distribution]
The Gumbel distribution with mean $\mu$ and parameter $\beta$ has the following probability density function:
\begin{align}
    \text{Gumbel}(\mu,\beta)(x) = \frac{1}{\beta} e^{-e^{-(x-\mu)/\beta}}
\end{align}
\end{definition}
We will make use of the following properties of the Gumbel Distribution:

\begin{lemma}[Gumbel Distribution Properties, \citep{chattamvelli2021gumbel}]
\label{lem:gumbel-properties}
The following are true:
\begin{itemize}
    \item The mean of a Gumbel distribution is $\mu + \beta\gamma$\footnote{ $\gamma\approx 0.577$ is the Euler-Mascheroni constant.}
    \item The moment generating function (MGF) of the Gumbel$(\mu,\beta)$ distribution is:
    \begin{align}
        M(t) = \Gamma(1-\beta t)e^{\mu t}
    \end{align}
    where $\Gamma$ is the Gamma function.
    \item We can easily 
    sample a Gumbel random variable of mean $\mu$ and parameter $\beta$ by using the uniform distribution:
    \begin{align}
        X = \mu - \beta \ln (-\ln (U)),\quad\text{ where } U\sim \text{unif}([0,1])
    \end{align}
    \item Let $M_1,...,M_n$ be $(\mu,\beta)$ independent Gumbel random variables. Then, $$M:=\max_{i\in [n]}M_i$$ 
    is a $(\mu + \beta\ln n, \beta)$ Gumbel random variable. 
\end{itemize}
\end{lemma}
Next, we present the well-known \textbf{Gumbel Max Trick}, an alternative way to sample from a softmax distribution:

\begin{lemma}[Gumbel-Max-Trick, \citep{huijben2022review}]
\label{lemma:gumbel-max-trick}
Let $x_1,...x_n$ be real numbers and consider the softmax categorical distribution $p$ where
$$
p_i = \frac{\exp(x_i)}{\sum\limits_{k=1}^n \exp(x_i)}
$$
Consider sampling $n$ Gumbel random variables $G_1,...,G_n \sim \text{Gumbel}(0,1)$ and let 
$$
\widehat{i} \in \arg \max\limits_{i \in [n]}\{x_i + G_i\}
$$
Then $\widehat{i}$ is distributed according to $p$.
\end{lemma}

\subsection{Locality Sensitive Hashing}
In our theoretical exposition we will make extensive use of schemes for \textit{approximate nearest neighbor search}. A very successful such suite of algorithms with a long history \citep{andoni2014beyond, andoni2015practical} of provable theoretical guarantees is Locality Sensitive Hashing (LSH):

\begin{theorem}[Existence of LSH, \citep{gionis1999similarity, mussmann2017fast}]
\label{thm:LSH_construction}
Let $V \subseteq U$ be a set of size $n$ with a similarity
measure $\text{Sim}(\cdot,\cdot)$. Consider a hash family $H$ such that for scalars $S_1 > S_2$ and $p_1 > p_2$:
\begin{itemize}
    \item For any $x, y \in V$ where $\text{Sim}(x, y) \geq S_1$, $\Pr_{h\in H}[h(x) = h(y)] \geq p_1$
    \item For any $x, y \in V$ where $\text{Sim}(x, y) \leq S_2$, $\Pr_{h \in H}[h(x) = h(y)] \leq p_2$
\end{itemize}
This is called an \textbf{$(S_1,S_2,p_1,p_2)$-Locality Sensitive Hash Family}. Given such a family, one can construct a data structure which, given any query $q \in U$, does the following with high probability: if there exists some point $v \in V$ with $\text{Sim}(v,q) \geq S_1$, it returns a point $v' \in V$ with $\text{Sim}(v',q) \geq S_2$. If no point $v \in V$ exists with $\text{Sim}(v,q) \geq S_2$, it returns a negative answer $\perp$. Further, this can be done with $\widetilde{O}(n^{\rho})$ query time and $\widetilde{O}(n^{1+\rho})$ space where $\rho = \log p_1/
\log p_2 < 1$.
\end{theorem}

LSH also finds numerous applications in solving the Maximum Inner Product Search Problem (MIPS) as shown in \citet{neyshabur2015symmetric}, \citet{shrivastava2014asymmetric}, and others.

\subsection{Random Walks and some concentration bounds}
For our back-propagation algorithms we will make use of some elementary tools from the theory of Random Walks. 

\begin{definition}[Random Walks]
\label{def:random_walk_def}
Consider a state space $V = [n]$ and a weighted complete graph on $V$ with weights $w$ in $[0,1]$ such that for all $u \in V$ 
$$
\sum\limits_{v \in V}w_{uv} = 1
$$
This graph represents a \textbf{random walk} with transition matrix $P \in [0,1]^{n \times n}$, where $P_{ij} = w_{ij}$. $P$ is a stochastic matrix because its rows sum to $1$. In a random walk, we start at some vertex and choose a neighbor to jump to according to the probability distribution in $P$. The choice at each vertex conditioned on the previous transitions only depends on the vertex itself. This is known as the Markov Property.
\end{definition}
A first elementary observation is that if we start with a distribution $p$ over $V$ and we do a single step in the random walk, we can obtain the resulting distribution by multiplying the original distribution with $P^T$:

\begin{lemma}[Single Step Random Walk Transition]
\label{lem:single_step_random_walk}
Consider a distribution $p \in \Delta(n)$\footnote{$\Delta(n) := \{x\in\mathbb{R}^n\mid x\geq 0,||x||_1 = 1\}$ is the probability simplex over $[n]$.} over $V$. Then, the quantity $P^T\cdot p$ gives the distribution over $V$ after one step of the random walk.
\end{lemma}
\begin{proof}
Let $q$ be the distribution after one step. Let $v \in V$. We have by law of total probability that:
$$
q(v) = \sum\limits_{u \in V}p(u)\cdot w_{uv} = (P^T p)_v
$$
\end{proof}
Finally, we state a well-known concentration result about independent random variables:
\begin{lemma}[Hoeffding Bound]
\label{lemma:hoeffding}
Let $X_1,...,X_n$ be independent random variables where $a_i \leq X_i \leq b_i$ almost surely. Let $S_n := X_1+\cdots+X_n$. We have that:
$$
\Pr\left[\left|S_n-\mathbb{E}[S_n]\right| \geq t\right] \leq 2\exp\left(-\frac{2t^2}{\sum\limits_{i=1}^n (b_i-a_i)^2}\right)
$$
\end{lemma}

\section{Proof of Lemma \ref{lemma:large-gumbel-bound}}
\label{appendix:m_bound_proof}
We prove Lemma \ref{lemma:large-gumbel-bound}. Our proof deviates from the proof of \citet{mussmann2017fast} in that it uses a MGF-based argument, which we believe is cleaner.
\begin{lemma}[Reminder]
In the context of Algorithm \ref{alg:lazy-gumbel}, we have that:
\begin{align}
    \mathbb{E}\left[m\right] \leq \frac{n}{k}
\end{align}
\end{lemma}
\begin{proof}
By Lemma \ref{lem:gumbel-properties}, we can generate Gumbel$(0,1)$ random variables as follows: Let $U_j$ be uniform in $[0,1]$. Then:
\begin{align}
    G_{ij} = -\ln(-\ln(U_j))
\end{align}
is distributed according to $\text{Gumbel}(0,1)$. We want $G_{ij} > B$, which implies that:
\begin{align}
    -\ln(-\ln(U_j)) > B \iff U_j > \exp(-\exp(-B))
\end{align}
So the number of points for which the Gumbel noise exceeds $B$ is distributed according to the Binomial distribution with parameters $n-k$ and $1-\exp(-\exp(-B))$. If we condition on $M := \max\limits_{j \in S_i}\{q_i^T k_j + G_{ij}\}$, we have that:
\begin{align}
    \mathbb{E}\left[m \mid M\right] &= (n-k)(1-\exp(-\exp(-B)))\\
    &\leq n\exp(-B)
\end{align}
where the last inequality follows by $e^{-x} \geq 1-x$:
$$
1-\exp(-\exp(-B)) \leq 1-(1-\exp(-B)) = \exp(-B)
$$
Now we can bound $\mathbb{E}[n\exp(-B)]$ by using the MGF of the Gumbel distribution (see Lemma \ref{lem:gumbel-properties}). Let $M' := \max_{j \in S_i}G_{ij}$. Recall by Lemma \ref{lem:gumbel-properties} that $M'$ is a Gumbel random variable with $\mu_{M'} = \log k$ and $\beta_{M'}=1$. Let $f_{M'}(t) = \mathbb{E}[e^{tM'}]$ be its moment generating function. We know that:
\begin{align}
    f_{M'}(t) = \Gamma(1-t) \cdot e^{(\log k)t} = \Gamma(1- t)\cdot k^{t}
\end{align}
This allows us to write:
\begin{align}
\mathbb{E}[n\exp(-B)] 
&= n\cdot \mathbb{E}[\exp(-B)] \\
&= n\cdot \mathbb{E}[\exp(S_{\min}-M)]\\
&= n\cdot \mathbb{E}[\exp(S_{\min}-\max\limits_{j\in S_i}\{Z_{ij} + G_{ij}\})]\\
\label{eq:remove_z_step}
&\leq n\cdot \mathbb{E}[\exp(S_{\min}-\min_{j\in S_i}Z_{ij} - M']\\
&= n\cdot \mathbb{E}[\exp(-M')]\\
&= n\cdot f_{M'}(-1)\\
\label{eq:gamma_last_step}
&= \frac{n}{k}
\end{align}
where inequality \ref{eq:remove_z_step} follows because 
$$
\max_{j \in S_i}\{Z_{ij} + G_{ij}\} \geq \min_{j\in S_i}\{Z_{ij}\} + \max_{j\in S_i} G_{ij} = S_{\min} + M'
$$
For a quick proof of this statement, let $\widehat{j} := \arg\min_{j \in S_i} Z_{ij}$ and $\widetilde{j} := \arg\max_{j \in S_i} G_{ij}$. Also let $j^* := \arg \max_{j \in S_i}\{Z_{ij} + G_{ij}\}$. Then we have:
\begin{align}
    \max\limits_{j\in S_i}\{Z_{ij} + G_{ij}\} &= Z_{ij^*} + G_{ij^*} \\
    &\geq Z_{i\widetilde{j}} + G_{i\widetilde{j}}\\
    &\geq Z_{i\widehat{j}} + G_{i\widetilde{j}}\\
    &= S_{\min} + M'
\end{align}
Finally \ref{eq:gamma_last_step} follows because $\Gamma(2) = 2! = 1$. Now, via law of total expectation we finally get:
\begin{align}
    \mathbb{E}[m] = \mathbb{E}_M\left[\mathbb{E}\left[m \mid M\right]\right] \leq \mathbb{E}_M\left[\frac{n}{k}\right] = \frac{n}{k}
\end{align}
\end{proof}

\section{$k$NN Attention via Concentric LSH}
\label{section:lsh_attn}
In the main paper, we abstracted away the specific $k$NN method used to obtain the top-$k$ key vectors $k_j$ for every query vector $q_i$. In this section we cover a method for solving the $k$-MIPS problem in sub-linear time per query that has sound theoretical guarantees. In the context of Algorithm \ref{alg:knn_attn}, this method could substitute the $k$NN index $H$.

This approach in question was proposed by \cite{mussmann2017fast} and it uses a concentric LSH construction to get an approximation to this problem. First, let us define the approximate version of the $k$-MIPS problem, as is proposed in \cite{mussmann2017fast}:
\begin{definition}[Approximate $k$-MIPS]
We say that a set $S_i$ is an approximate top-$k$ inner product solution if $|S_i| = k$ and there exists a constant $c$ such that:
\begin{align}
    \max\limits_{j \notin S_i}q_i^T k_j - \min\limits_{j \in S_i}q_i^T k_j < c
\end{align}
\end{definition}
If we are able to generate an approximate solution $S_i$ instead of an exact one, we have to lower our threshold $B$ to $M-S_{\min} - c$ in Algorithm \ref{alg:lazy-gumbel}, which in turns implies that $E[m] \leq \sqrt{n}\cdot e^c$. This remains sublinear in $n$ because $c$ is a constant.

The solution to the approximate version of the problem is constructed using LSH. Specifically, we build a sequence of \(O(\text{polylog}(n, d))\) LSH data structures, each with concentric approximation radii, and hash all the key vectors \(k_j \in \mathbb{R}^d\) into them. For each query vector \(q_i\), we hash it across all these data structures and identify the first pair of consecutive LSH structures, \(D_i\) and \(D_{i+1}\), where \(D_{i+1}\) contains more than \(\sqrt{n}\) points in the buckets corresponding to \(q_i\), while \(D_i\) contains fewer than \(\sqrt{n}\) points. For further details, see \cite{mussmann2017fast}. The following theorem ultimately holds:

\begin{theorem}[\cite{mussmann2017fast}]
\label{thm:sublinear-lsh-kmips}
Let $0 < \rho < 1$ be a constant. There exists an algorithm for solving the approximate version of $k$-MIPS on any single query $q$ with probability at least $1-1/n^2$ by using an explicit concentric LSH construction. The algorithm takes $O(dn^{1+\rho}\cdot \text{polylog}(n, d))$ pre-processing time/space, and $O(\sqrt{n} + n^\rho \cdot \text{polylog}(n, d))$ 
time/space per query. 
\end{theorem}

\begin{figure}[h]
    \centering
    \includegraphics[width=0.6\linewidth]{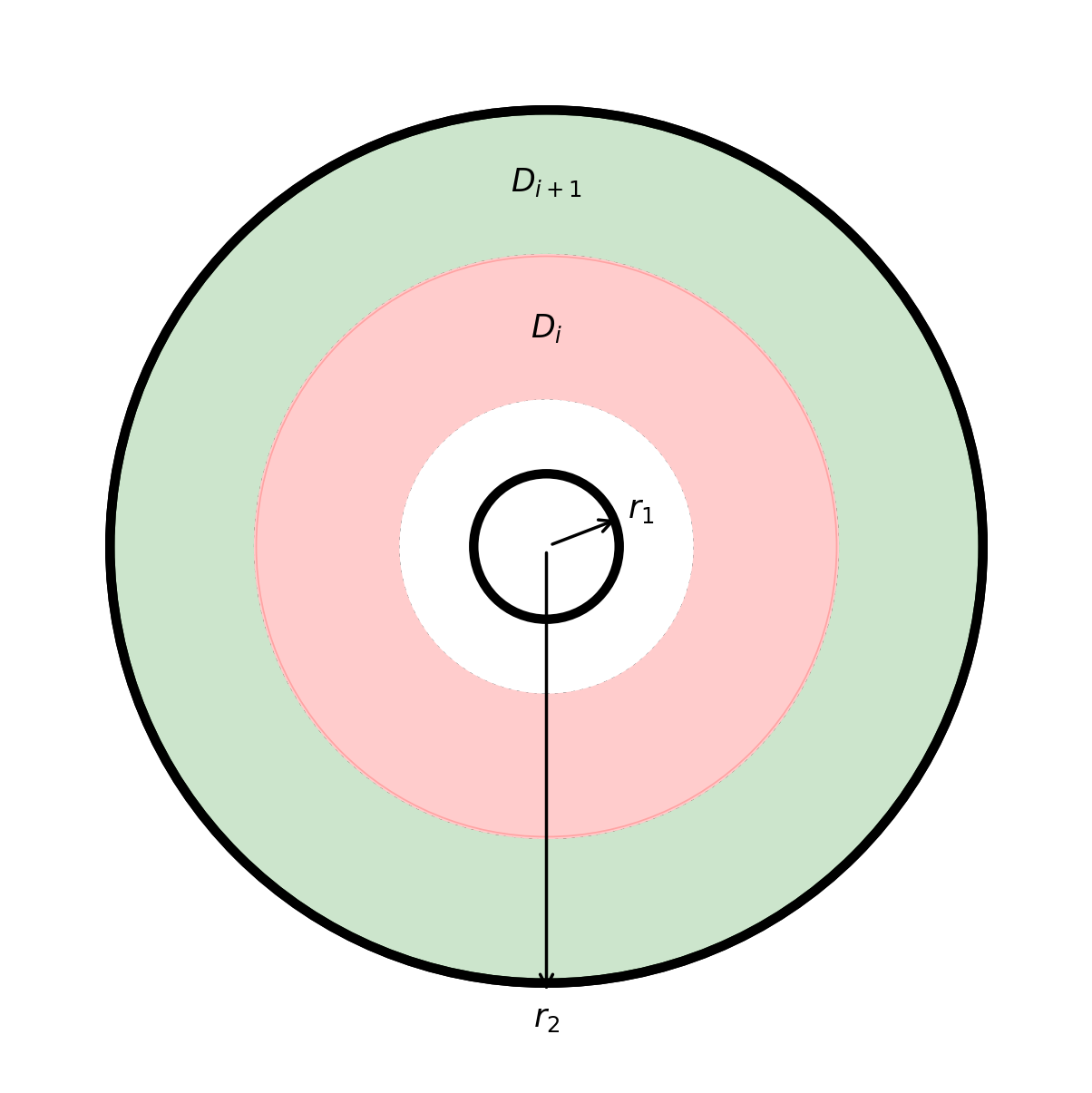}
    \caption{An illustration of the concentric LSH construction of \cite{mussmann2017fast} In the $D_{i+1}$ band we find at least $\sqrt{n}$ points and in the $D_i$ band we find fewer than $\sqrt{n}$ points. }
    \label{fig:enter-label}
\end{figure}

\begin{remark}[The role of $\rho$]
The choice of $\rho < 1$ allows us to compute $S_i$ in sublinear time for each $i \in [n]$. The value of $\rho$ is determined by the radii gaps in the concentric construction. Our algorithm for computing $S_i$ is sublinear in $n$ because $\rho < 1$. Depending on the particular input dataset, we could have $\rho \leq 1/2$, which which case $f(n,\sqrt{n}) = \widetilde{O}(\sqrt{n})$ in the context of Theorem 6.
\end{remark}

\subsection{A complete algorithm based on solving $k$-MIPS}
We now have a complete algorithm to estimate self-attention with provable guarantees that is based on the solving $k$-MIPS problem for every query vector $q_i$. If we combine the boosted estimator approach of Theorem \ref{thm:approx-attn} with the Lazy Gumbel Sampling Technique of Algorithm \ref{alg:lazy-gumbel} and the $k$-MIPS LSH technique of Theorem \ref{thm:sublinear-lsh-kmips}, we arrive at the following theorem. We give pseudocode for the resulting algorithm, as Algorithm \ref{alg:kmips_lsh_attn}:

\begin{theorem}
\label{thm:attn_lazy_gumbel_lsh}
Let $\varepsilon >0$ and $\delta > 0$ be small positive constants. There exists an algorithm that can estimate Self-Attention in the same way as Theorem \ref{thm:approx-attn} and fail with probability at most $\delta + 1/n$. The algorithm's time complexity is shown in the table below, where $\rho \in (0,1)$ is a fixed constant.

\begin{table}[H]
\centering
\begin{tabular}{ c|cc }
 & Pre-Processing & Main Computation\\ \hline \\[-1em]
 Complexity  & $\widetilde{O}(dn^{1+\rho})$ & $\widetilde{O}\left(n^{1+\max\{1/2,\rho\}}\cdot d\cdot \varepsilon^{-2}\log(1/\delta)\right)$\\
\end{tabular}
\end{table}

\end{theorem}

\begin{algorithm}[h]
\caption{Approximating Self-Attention using concentric LSH $k$-MIPS solver}\label{alg:kmips_lsh_attn}
\begin{algorithmic}[1]
\State \textbf{Inputs: }$Q,K,V \in \mathbb{R}^{n\times d}$, error parameter $\varepsilon > 0$, confidence parameter $\delta > 0$
\vspace{2mm}
\State $H \gets$ Create Concentric LSH data structures for solving $k$-MIPS, as in \cite{mussmann2017fast}
\vspace{2mm}
\State Let $\widehat{O} \in \mathbb{R}^{n\times d}$ be our output.
\For{$i \in [n]$}
    \State $S_i \gets$ Query $H$ for the $\sqrt{n}$ indices $j \in [n]$ with the \textit{approximate} largest values of $q_i^T k_j$
    \For{$j \in [d]$}
        \State $\widehat{O_{ij}} \gets$ Median-Of-Means with Algorithm 1 as sampler $\gets (\sqrt{n},q_i,K,S_i)$.
    \EndFor
\EndFor
\State \Return $\widehat{O}$
\end{algorithmic}
\end{algorithm}

\begin{proof}
Let $k = \sqrt{n}$. Suppose we construct the concentric LSH data structure according to Theorem \ref{thm:sublinear-lsh-kmips}. This takes $\widetilde{O}(dn^{1+\rho})$ time. Let us condition on the event that for all queries $q_i$ the data structure provides a correct approximate answer $S_i$ to the $k$-MIPS problem. This happens with probability at least $1 - 1/n$ by union bound over all $n$ queries. Now we use our sets $S_i$ in Algorithm \ref{alg:lazy-gumbel} to sample from $D_i$. Since retrieving $S_i$ takes $f(n,k) = O(\sqrt{n} + n^\rho \cdot \text{polylog}(n))$ time and space, Theorem \ref{thm:runtime-analysis-lazy-gumbel} dictates that sampling from $D_i$ also takes $\widetilde{O}(n^{\max\{1/2,\rho\}})$ time and space. Thus, in the context of Theorem \ref{thm:approx-attn} we have that $T = \widetilde{O}(n^{\max\{1/2,\rho\}})$. Substituting back gives us the desired runtime and failure probability guarantees. 
\end{proof}

\section{Derivation of the Self-Attention Gradients}
\label{sec:grad-derivation}
Suppose we have a scalar function $\phi$ that represents the loss when training our neural network after computing the output $O$: $\ell = \phi(O)$. Suppose that we have calculated $\frac{\partial \phi}{\partial O_{ij}}$ for all $i \in [n], j \in [d]$ and stored it in an matrix $D^O \in \mathbb{R}^{n \times d}$. Now we will calculate the remaining derivatives by using the chain rule. A similar calculation is also done in the Appendix of \cite{dao2022flashattention}.

\paragraph{Calculating $\frac{\partial \phi}{\partial V_{ij}}$}
All these calculations just use the chain rule. One can simply draw a tree of dependencies and use it to perform the derivation. $\phi$ depends on $O_{ij}$ and $O_{ij}$ depends on all $V_{rj}$, so:
$$
\frac{\partial \phi}{\partial V_{ij}} 
= \sum\limits_{r=1}^n \frac{\partial \phi}{\partial O_{rj}}\cdot \frac{\partial O_{rj}}{\partial V_{ij}}
= \sum\limits_{r=1}^n D^O_{rj}\frac{\partial O_{rj}}{\partial V_{ij}}
$$
Now, we calculate that:
$$
\frac{\partial O_{rj}}{\partial V_{ij}} = \frac{\partial}{\partial V_{ij}}\sum\limits_{k=1}^n P_{rk}V_{kj} = P_{ri}
$$
so that gives:
\begin{align}
    \frac{\partial \phi}{\partial V_{ij}} =\sum\limits_{r=1}^n D^O_{rj}P_{ri} = \sum\limits_{r=1}^n P^T_{ir} D^O_{rj}
\end{align}
Thus, we can write the result succinctly:
\begin{align}
    D^V = P^T \cdot D^O
\end{align}

\paragraph{Calculating $\frac{\partial \phi}{\partial Q_{ij}}$}
To do this, we will first calculate $\frac{\partial \phi}{\partial P_{ij}}$ and $\frac{\partial \phi}{\partial S_{ij}}$, where $S = QK^T$.
\begin{itemize}
    \item First, each $O_{ij}$ depends on all $P_{ik}$, so the chain rule gives:
    $$
    \frac{\partial \phi}{\partial P_{ij}} = \sum\limits_{k=1}^d \frac{\partial \phi}{\partial O_{ik}}\cdot \frac{\partial O_{ik}}{\partial P_{ij}} = \sum\limits_{k=1}^d D^O_{ik}\frac{\partial O_{ik}}{\partial P_{ij}}
    $$
    We can calculate that:
    $$
    \frac{\partial O_{ik}}{\partial P_{ij}} = V_{jk}
    $$
    and so:
    \begin{align}
    D^P_{ij} = \frac{\partial \phi}{\partial P_{ij}} = \sum\limits_{k=1}^d D^O_{ik} V_{jk} = \langle D^O_{i,:}, V_{j,:}\rangle
    \end{align}
    for all $i \in [n], j\in [n]$.
    \item Now recall that $P_{ij} = \frac{\exp(S_{ij})}{L_i}$, so $P_{ij}$ depends on all $S_{ik}$ for $k=1,...,n$. Thus:
    \begin{align*}
    \frac{\partial \phi}{\partial S_{ij}} = \sum\limits_{k=1}^n\frac{\partial \phi}{\partial P_{ik}}\cdot \frac{\partial P_{ik}}{\partial S_{ij}} &= \sum\limits_{k=1}^n D^P_{ik}\cdot \frac{\partial P_{ik}}{\partial S_{ij}} \\
    & = D^P_{ij}\cdot \frac{\partial P_{ij}}{\partial S_{ij}} + \sum\limits_{k=1, k\neq j}^n D^P_{ik}\cdot \frac{\partial P_{ik}}{\partial S_{ij}}
    \end{align*}
    We now calculate seperately the two cases by using the quotient rule:
    \begin{itemize}
        \item $k \neq j$: 
        \begin{align*}
        \frac{\partial P_{ik}}{\partial S_{ij}} = \frac{\partial}{\partial S_{ij}}\frac{\exp(S_{ik})}{\sum\limits_{r=1}^n \exp(S_{ir})} &= -\exp(S_{ik})\cdot \frac{\exp(S_{ij})}{\left(\sum\limits_{r=1}^n \exp(S_{ir})\right)^2} \\
        &= -P_{ik}P_{ij}
        \end{align*}
        \item $k = j$:
        \begin{align*}
        \frac{\partial P_{ij}}{\partial S_{ij}} = \frac{\partial}{\partial S_{ij}}\frac{\exp(S_{ij})}{\sum\limits_{r=1}^n \exp(S_{ir})} &= \frac{\exp(S_{ij}) \sum\limits_{r=1}^n \exp(S_{ir}) - \exp(S_{ij})\exp(S_{ij})}{\left(\sum\limits_{r=1}^n \exp(S_{ir})\right)^2} \\
        &= P_{ij} - P_{ij}^2
        \end{align*}
    \end{itemize}
    Now we can put it all together:
    \begin{align*}
    \frac{\partial \phi}{\partial S_{ij}} &= D^P_{ij}\cdot \frac{\partial P_{ij}}{\partial S_{ij}} + \sum\limits_{k=1, k\neq j}^n D^P_{ik}\cdot \frac{\partial P_{ik}}{\partial S_{ij}} \\
    &= D^P_{ij}\cdot(P_{ij} - P_{ij}^2) - \sum\limits_{k=1, k\neq j}^n D^P_{ik}\cdot P_{ik}P_{ij} \\
    &= D^P_{ij}\cdot P_{ij} - \sum\limits_{k=1}^n D^P_{ik}\cdot P_{ik}P_{ij} \\
    &= P_{ij}\left(D^P_{ij} - \langle D^P_{i,:}, P_{i,:}\rangle\right)
    \end{align*}
\end{itemize}

\noindent
Now finally, for $i \in [n], j\in [d]$, $Q_{ij}$ influences $S_{ik}$ for all $k \in [n]$, so:
\begin{align}
\frac{\partial \phi}{\partial Q_{ij}} &= \sum\limits_{k=1}^n \frac{\partial \phi}{\partial S_{ik}}\frac{\partial S_{ik}}{\partial Q_{ij}} \\
&=\sum\limits_{k=1}^n P_{ik}\left(D^P_{ik} -\langle D^P_{i,:}, P_{i,:}\rangle\right)K_{kj}
\end{align}

\paragraph{Calculating $\frac{\partial \phi}{\partial K_{ij}}$} We know that $K_{ij}$ influences $S_{ki}$ for $k \in [n]$, so:
\begin{align}
    \frac{\partial \phi}{\partial K_{ij}} &= \sum\limits_{k=1}^n \frac{\partial \phi}{\partial S_{ki}}\frac{\partial S_{ki}}{\partial K_{ij}} \\ 
    &= \sum\limits_{k=1}^n P_{ki}\left(D^P_{ki} - \langle D^P_{k,:}, P_{k,:}\rangle\right)Q_{kj}
\end{align}

\section{Estimating $D^Q$}
\label{section:estimating_dq}
In this section we give an efficient algorithm for estimating $D^Q$. This algorithm is based on our $k$NN-Attention framework. Recall that we found that:
$$
D^Q_{ij} = \sum\limits_{k=1}^n P_{ik}\left(D^P_{ik} -\langle D^P_{i,:}, P_{i,:}\rangle\right)K_{kj}
$$
We can write this expression as an expectation with respect to the distribution $D_i$:
\begin{align}
\frac{\partial \phi}{\partial Q_{ij}} &= \mathbb{E}_{k\sim D_i}\left[D^P_{ik}K_{kj}\right] - \mathbb{E}_{k\sim D_i}\left[K_{kj}\cdot \mathbb{E}_{s\sim D_i}[D^P_{is}]\right] \\
&=\underbrace{\mathbb{E}_{k\sim D_i}\left[D^P_{ik}K_{kj}\right]}_{E_1} - \underbrace{\mathbb{E}_{k\sim D_i}\left[K_{kj}\right]}_{E_2}\cdot \underbrace{\mathbb{E}_{s\sim D_i}[D^P_{is}]}_{E_3}
\end{align}
This allows us to use any of our softmax expectation estimators. We choose the Median-Of-Means estimator for the purposes of a clean analysis. We just have to do it three times and ensure that the terms we take expectations over are efficiently computable. Indeed, because
\begin{align}
D^P_{ik} = \langle D^O_{i,:}, V_{k,:}\rangle,
\end{align}
we can compute all three of those expectations in sublinear time! Let $\widehat{E_1}, \widehat{E_2}, \widehat{E_3}$ be the estimates we produce. Then, almost identically to the error analysis we did for the forward pass, we get an $(\varepsilon,\delta)$-\textit{additive} estimate for $E_i$, where $i\in\{1,2,3\}$\footnote{In Theorem \ref{thm:approx-attn} we used a multiplicative approximation. To get the additive approximation guarantee we need $O(\varepsilon^{-2}\log(1/\delta)\cdot \text{Var}[\widehat{O}_{ij}])$ samples, where $\text{Var}[\widehat{O}_{ij}] \leq B^2 = O(\text{polylog}(n))$.}
\begin{align}
\Pr\left[|\widehat{E_1}-E_1| \geq \varepsilon\right] \leq \frac{\delta}{3}\\
\Pr\left[|\widehat{E_2}-E_2| \geq \varepsilon \right] \leq \frac{\delta}{3}\\
\Pr\left[|\widehat{E_3}-E_3| \geq \varepsilon \right] \leq \frac{\delta}{3}
\end{align}
And so, putting these three together and using the union bound we get that with probability at least $1-\delta$ it holds that:
\begin{align}
\left|\widehat{E_1}- \widehat{E_2}\cdot \widehat{E_3} - E_1 + E_2 \cdot E_3\right| &\leq \left|\widehat{E_1}-E_1\right| + \left|\widehat{E_2}\cdot \widehat{E_3}-E_2\cdot E_3\right|\\
&=\left|\widehat{E_1}-E_1\right| + \left|\widehat{E_2}\cdot \widehat{E_3}-\widehat{E_2}\cdot E_3 + \widehat{E_2}\cdot E_3-E_2\cdot E_3\right|\\
&\leq \left|\widehat{E_1}-E_1\right| + \widehat{E_2}\left|\widehat{E_3}-E_3\right| + 
E_3\left|\widehat{E_2}-E_2\right|\\
&\leq \varepsilon + \varepsilon \cdot\widehat{E_2} + \varepsilon E_3\\
&\leq \varepsilon + \varepsilon(E_2 + \varepsilon) + \varepsilon E_3\\
\label{eq:error_Q_raw}
&= \varepsilon + \varepsilon^2 + \varepsilon(E_2 + E_3)
\end{align}
In order to bound the variance of our estimators, we need to assume some bounds on the inputs, analogously to $||V||_\infty \leq B = O(\lg n)$ from Theorem \ref{thm:approx-attn}. First, we assume that $||K||_\infty \leq B_K = O(\text{polylog}(n))$. Second, we have that $||D^P||_\infty \leq dB\cdot ||D^O||_\infty = O(\text{polylog}(n))$ if $d = O(\log n)$. This also gives that $||D^P \circ K||_\infty \leq B_K\cdot ||D^P||_\infty = O(\text{polylog}(n))$.

These assumptions are reasonable within the context of the hardness results proved for the attention mechanism and the computation of its gradients\footnote{Another motivation for assumming an upper bound on the norm of $D^O$ is to avoid the phenomenon of exploding gradients in training neural networks.} \citep{alman2024fast, alman2024fine}. Given these assumptions, we can also bound the error more compactly. Starting from Equation \ref{eq:error_Q_raw}, we get:
$e_Q \leq O(\varepsilon) + \varepsilon(B_K + dB\cdot ||D^O||_\infty) = O(\varepsilon\cdot \text{polylog}(n))
$. As a result, we arrive at the following theorem:

\begin{theorem}
\label{thm:dq_approximation}
Assume that $||K||_\infty = O(\text{polylog}(n)), d = O(\log n)$ and $||D^O||_\infty = O(\text{polylog}(n))$. There exists a sub-quadratic algorithm that takes as input $Q,K,V,D^O \in \mathbb{R}^{n\times d}$ and outputs a matrix $\widehat{D}^Q \in \mathbb{R}^{n\times d}$ such that:
\begin{align}
    \left|\left|\widehat{D}^Q-D^Q\right|\right|_\infty \leq O(\varepsilon\cdot \text{polylog}(n))
\end{align}
This algorithm is shown as Algorithm \ref{alg:dq_estimation}.
\end{theorem}

\begin{algorithm}[!h]
\caption{Estimating $D^Q$}\label{alg:dq_estimation}
\begin{algorithmic}
\Procedure{Estimate-$E_1$}{$Q,K,V,D^O, S_i, i, j,\varepsilon, \delta$}
    \State $F \gets \{\langle D^O_{i,:},V_{k,:}\rangle\cdot K_{kj}\}_{k=1}^n \in \mathbb{R}^{n\times 1}$ \Comment{$F$ \textit{will not be materialized.}}
    \State $\widehat{E_1} \gets $ Median-Of-Means with Lazy Gumbel Sampling $\gets $ $Q,K,F, S_i, \varepsilon,\delta$
    \State \Return $\widehat{E_1}$
\EndProcedure
\vspace{1mm}
\Procedure{Estimate-$E_2$}{$Q,K, S_i, i, \varepsilon, \delta$}
    \State $\widehat{E_2} \gets $ Median-Of-Means with Lazy Gumbel Sampling $\gets Q,K,K_{:,j},S_i,\varepsilon,\delta$.
    \State \Return $\widehat{E_2}$.
\EndProcedure
\vspace{1mm}
\Procedure{Estimate-$E_3$}{$Q,K,V,D^O, S_i, i, \varepsilon, \delta$}
    \State $F \gets \{\langle D^O_{i,:}, V_{k,:}\rangle\}_{k=1}^n \in \mathbb{R}^{n\times 1}$ \Comment{$F$ w\textit{ill not be materialized.}}
    \State $\widehat{E_3} \gets $ Median-Of-Means with Lazy Gumbel Sampling $\gets Q,K,F, S_i, \varepsilon,\delta$
    \State \Return $\widehat{E_3}$
\EndProcedure
\vspace{1mm}
\State \textbf{Input: }$D^O \in \mathbb{R}^{n\times d}$, $Q,K,V \in \mathbb{R}^{n\times d}$, parameters $\varepsilon,\delta > 0$
\vspace{1mm}
\State Let $\widehat{D}^Q \in \mathbb{R}^{n\times d}$ be our output.
\For{$i \in [n]$}
    \State $S_i \gets$ $\sqrt{n}$ values $t \in [n]$ of the largest $q_i^T k_t$ via LSH or $k$NN.
    \For{$j \in [d]$}
        \State $\widehat{E_1} \gets$ \Call{Estimate-$E_1$}{$Q,K,V,D^O$,$S_i, i,j,\varepsilon, \delta$}
        \State $\widehat{E_2} \gets$ \Call{Estimate-$E_2$}{$Q,K$,$S_i, i,\varepsilon, \delta$}
        \State $\widehat{E_3} \gets$ \Call{Estimate-$E_3$}{$Q,K,V,D^O$, $S_i, i,\varepsilon, \delta$}

        \State $\widehat{D}^Q_{ij} \gets \widehat{E_1} - \widehat{E_2}\cdot \widehat{E_3}$
    \EndFor
\EndFor
\State \Return $\widehat{D}^Q$
\end{algorithmic}
\end{algorithm}

\section{Estimating $D^K$}
\label{section:estimating_dk}
Finally, we turn to estimating $D^K$. Our earlier calculations show that
$$
    \frac{\partial \phi}{\partial K_{ij}}  = 
    \sum\limits_{k=1}^n P_{ki}\left(D^P_{ki} - \langle D^P_{k,:}, P_{k,:}\rangle\right)Q_{kj}
$$
We can break up this sum into two terms:
\begin{align}
\frac{\partial \phi}{\partial K_{ij}}  = 
    \underbrace{\sum\limits_{k=1}^n P_{ki}D^P_{ki}Q_{kj}}_{A_{ij}} - \underbrace{\sum\limits_{k=1}^n P_{ki}\langle D^P_{k,:}, P_{k,:}\rangle\cdot Q_{kj}}_{B_{ij}}
\end{align}
We will estimate both terms separately:
\subsection{Estimating $A_{ij}$} For $i \in [n]$ and $ j \in [d]$, we have:
\begin{align}
A_{ij} = \sum\limits_{k=1}^n P_{ki} D^P_{ki}Q_{kj} &= \sum\limits_{k=1}^n P_{ki}Q_{kj}\cdot \langle D^O_{k,:}, V_{i,:}\rangle \\
&= \sum\limits_{k=1}^n P^T_{ik}\cdot Y_{kj}^{(i)}
\end{align}
where $Y_{kj}^{(i)} := Q_{kj}\cdot \langle D^O_{k,:}, V_{i,:}\rangle$. So we can write:
\begin{align}
A_{ij} = \underbrace{(P^T)_{i,:}}_{1\times n} \cdot \underbrace{Y^{(i)}_{:,j}}_{n\times 1}
\end{align}
We will use our familiar Markov Chain estimation method from Algorithm \ref{alg:mcmc_alg} to calculate this quantity. However, in this case we only care about estimating the $i$-th entry in the vector $(P^T)\cdot Y^{(i)}_{:,j}$, which we can do by performing $O(\lg n\cdot \varepsilon^{-2})$ simulations. Ultimately, by following the same analysis as in Algorithm \ref{alg:mcmc_alg}, we are able to estimate $A_{ij}$ with probability at least $1-\frac{1}{n}$ and error:
\begin{align}
    \left|\widehat{A}_{ij}-A_{ij}\right| &\leq \varepsilon \langle Y^{(i)}_{:,j}, 1^n \rangle + 2\varepsilon n M_j^{(i)} \\
    &= \varepsilon \sum\limits_{k=1}^n Q_{kj}\cdot \langle D^O_{k,:}, V_{i,:}\rangle+ 2\varepsilon n M_j^{(i)}
\end{align}
where 
$$
M_j^{(i)} = -\min\limits_{\substack{k \in [n]\\Y^{(i)}_{kj} \leq 0}}Y^{(i)}_{kj}
$$
\begin{remark}
Because we would need to calculate all $n^2$ values of $M_j^{(i)}$, we will instead use a single upper bound $M \geq M_j^{(i)}$ for all $(i,j) \in [n] \times [d]$ for this algorithm. We assume that we know a large enough $M$ in advance and that $M = O(\text{polylog}(n))$.
\end{remark}

In the next paragraphs, we will tackle some implementation issues that arise in this approach. We did not see these issues when estimating $D^V$, and because they make the algorithm a lot more complicated, we left them for last. 

\paragraph{Pre-calculating the normalizing factors} 
We need to pre-calculate the normalizing sums $N_j^{(i)} = \langle Y^{(i)}_{:,j}, 1^n\rangle + nM$ for all $(i,j) \in [n] \times [d]$. Naively, it takes $O(n^2 d)$ time to calculate all those sums. However, with some preprocessing we can take the time down to $O(nd^2)$. First, observe that we have:
\begin{align}
N_j^{(i)} = nM + \langle Y^{(i)}_{:,j}, 1^n\rangle &= nM+ \sum\limits_{k=1}^n Q_{kj}\cdot \langle D^O_{k,:}, V_{i,:}\rangle = nM+\langle V_{i,:}, \sum\limits_{k=1}^n Q_{kj}\cdot D^O_{k,:} \rangle
\label{eq:l1_norm_trick}
\end{align}
We can thus first pre-compute the $d$ vectors $\overrightarrow{E_j} = \sum_{k=1}^n Q_{kj}\cdot D^O_{k,:} \in \mathbb{R}^d$ for each $j\in[d]$ in $O(nd^2)$ time. Then, for each $i \in [n]$ and $j \in [d]$, we can  produce $N_j^{(i)}$ in $O(d)$ time by using Equation \ref{eq:l1_norm_trick}, bringing the total time complexity to $O(nd^2)$. 

\paragraph{Sampling according to $Y^{(i)}_{:,j} + M\cdot 1^n$ efficiently} Unfortunately, because we are now estimating $A_{ij}$ individually for all $(i,j) \in [n]\times [d]$, we cannot spend $O(n)$ time to generate each sample. We need to generate samples in sublinear time with some pre-processing. This seems intuitively difficult at first because we have $O(nd)$ distributions over $[n]$ and each distribution requires $\Omega(n)$ time to sample one sample. However, we can take advantage of the structure between the distributions in order to reduce the pre-processing time. First, consider the following method of sampling from a distribution $[p_1,...,p_n]$:
\begin{enumerate}
    \item Compute the cumulative sums $s_i = \sum_{k=1}^i p_i$. We know that $s_1 = p_1$ and $s_n = 1$. 
    \item Pick some $x \sim \text{Unif}(0,1)$ uniformly at random from $(0,1)$.
    \item Find the interval $[p_i, p_{i+1}]$ for $i\in[1,n-1]$ in which $x$ falls in. That is, find the smallest $i$ for which $x \leq s_i$. We can do this in $O(\lg n)$ time using binary search.
    \item Output $i$.
\end{enumerate}

\begin{figure}[h]
    \centering
    \includegraphics[width=0.7\linewidth]{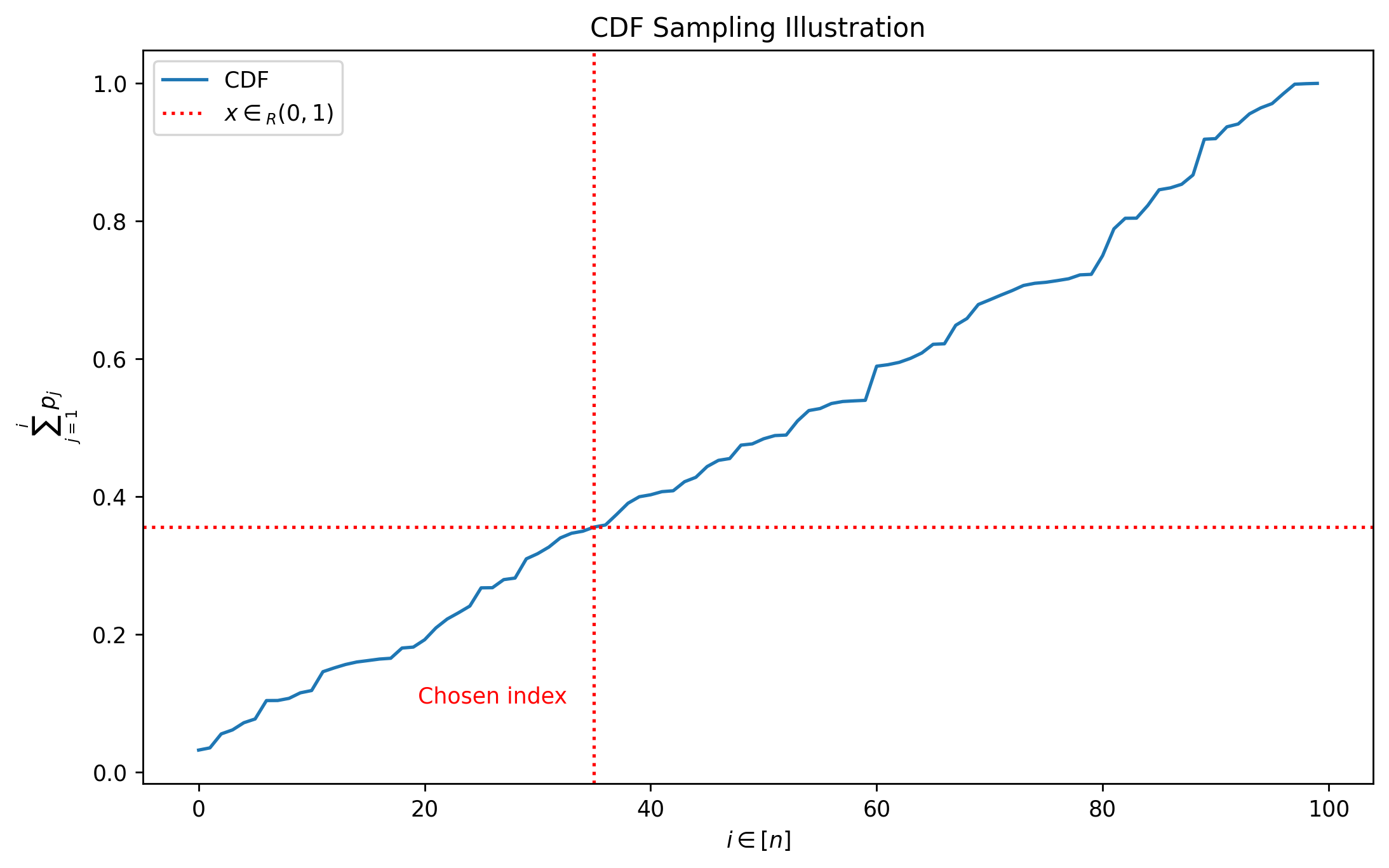}
    \caption{An illustration of the CDF sampling method: We form the CDF and then sample an index by choosing $x \in (0,1)$ and using binary search to find the corresponding bucket.}
    \label{fig:cdf_sampling}
\end{figure}

It is easy to see that this method outputs a value $i$ with probability $p_i$. If we applied this method naively we would still take $O(n^2 d)$ time because we'd have to calculate all the cumulative sums. However, the inner product structure again comes to our rescue:
\begin{align}
    Y^{(i)}_{kj} = \langle V_{i,:}, Q_{kj}\cdot D^O_{k,:}\rangle 
\end{align}
So, we can create $d$ cumulative sum tables $\Sigma_{j}$ for $j \in [d]$, each of which stores $n$ cumulative-sum $\mathbb{R}^d$ vectors as follows:
\begin{align}
    (\Sigma_j)_{\ell} = \sum\limits_{s=1}^\ell Q_{sj}\cdot D^O_{s,:} \in \mathbb{R}^d,\,\forall \ell\in [n]
\end{align}
This requires $O(nd^2)$ time and space to construct. Now, in order to sample with probability proportional to $Y^{(i)}_{kj} + M$ given that we know $N_j^{(i)}$, we sample $x_{ij} \sim \text{Unif}(0,1)$ and perform binary search to find the interval $x_{ij}$ belongs to. At that point, we can calculate the $O(\log n)$ necessary cumulative sums in $O(d)$ time each by using our pre-processing:
\begin{align}
    \sum\limits_{s=1}^\ell \left(Y^{(i)}_{kj} + M\right) = kM + \langle V_{i,:}, (\Sigma_j)_\ell\rangle
\end{align}
This allows us to sample in $O(d\lg n)$ time after a $O(nd^2)$ pre-processing. Our algorithm in total is included as part of Algorithm \ref{alg:dk_estimation_1}.

\paragraph{Sampling with respect to $D_i$}
Again, we cannot afford to sample from the softmax naively with $O(n)$ time. Thankfully, we know of a sublinear method that can allow us to sample from the softmax, with slightly super-linear pre-processing time: the Lazy-Gumbel Sampling method. We will omit the pre-processing details in the algorithm pseudocode.

\begin{algorithm}[h!]
\caption{Estimating $D^K$ -- Part 1: Computing $A$}\label{alg:dk_estimation_1}
\begin{algorithmic}[1]
\State \textbf{Input: }$Q,K,V,D^O \in \mathbb{R}^{n\times d}$, error parameter $\varepsilon > 0$
\vspace{1mm}
\For{$j \in [d]$}\Comment{Pre-Processing}
    \State Compute $\overrightarrow{E_j} = \sum_{k=1}^n Q_{kj}\cdot D^O_{k,:} \in \mathbb{R}^d$
    \State Compute the cumulative sums $(\Sigma_j)_{\ell} = \sum_{s=1}^\ell Q_{sj}\cdot D^O_{s,:} \in \mathbb{R}^{d}$ for all $\ell \in [n]$.
    \State Compute $\widehat{s} \approx P^T 1^n$ using Markov Chain simulations.
    \State Initialize a $k$NN index $H$.
\EndFor
\vspace{1mm}
\Procedure{Compute--$A$}{$Q,K,D^O,E,\Sigma,\varepsilon, H, \widehat{s}, M$}
    \State Let $N \gets 2\lg n \cdot \varepsilon^{-2}$
    \State $\widehat{A}\gets [0]^{n\times d}$ is the output.
    \For{$i \in [n]$}
        \State Query $H$ to get set $S_i$
        \For{$j \in [d]$}
            \State $N_j^{(i)} \gets \langle V_{i,:}, \overrightarrow{E_j}\rangle + nM$ \Comment{$O(d)$ time.}
            \For {$s \in [N]$}
                \State Sample $k \in [n]$ with probability $\propto Y^{(i)}_{kj} + M$ via binary search, $\Sigma_{j}$  and $N_j^{(i)}$
                \State Sample $\ell \in [n]$ with probability $P_{ik}$ via Lazy Gumbel Sampling, given $S_i$
                \If{$\ell = i$}
                    \State $\widehat{A}_{ij} \gets \widehat{A}_{ij} + 1$
                \EndIf
            \EndFor
            \vspace{1mm}
            \State $\widehat{A}_{ij} \gets \frac{1}{N}(\widehat{A}_{ij}\cdot N_j^{(i)}) - M\cdot \widehat{s}_i$
        \EndFor
    \EndFor

    \State \Return $\widehat{A}$
\EndProcedure
\end{algorithmic}
\end{algorithm}

\subsection{Estimating $B_{ij}$} For $(i,j) \in [n]\times [d]$, we first have: 
\begin{align}
    B_{ij} &= \sum\limits_{k=1}^n P_{ki}\cdot \langle D^P_{k,:}, P_{k,:}\rangle\cdot Q_{kj}\\
    &= \sum\limits_{k=1}^n P_{ki}X_{kj}
\end{align}
where $X_{kj} = \langle D^P_{k,:}, P_{k,:}\rangle\cdot Q_{kj}$. Notice that $X_{kj}$ takes $O(nd)$ time to naively compute, so we will first approximate it with $\widehat{X}_{kj}$. Observe that:
\begin{align}
    X_{kj} &= Q_{kj}\cdot\langle D^P_{k,:}, P_{k,:}\rangle\\ 
    &=  Q_{kj}\cdot\sum\limits_{s = 1}^n D^P_{ks}\cdot P_{ks} \\
    &=  Q_{kj}\cdot\mathbb{E}_{s\sim D_k}[D^P_{ks}] \\
    &= \mathbb{E}_{s\sim D_k}[Q_{kj}\cdot D^P_{ks}]\\ 
    &\approx \widehat{X}_{kj}
\end{align}
Let us approximate $\mathbb{E}_{s\sim D_k}[Q_{kj}\cdot(D^p)_{ks}] \approx \widehat{X}_{kj}$ using the Lazy Gumbel Sampling and Median-Of-Means method. This allows us to get for all $(k,j)\in [n]\times [d]$ with probability at least $1-\delta$ that:
\begin{align}
\left|\widehat{X}_{kj}-X_{kj}\right| \leq \varepsilon
\end{align}
\begin{remark}
To have an $o(n)$ bound for the variance, we have to assume (again) that $||X||_\infty = O(\text{polylog}(n))$. This follows from the assumption that $||Q||_\infty = O(\text{polylog}(n))$ and $||D^P||_\infty = O(\text{polylog}(n))$. The latter follows from $||D^O||_\infty = O(\text{polylog}(n))$. So the assumptions here are the same as in Theorem \ref{thm:dq_approximation}.
\end{remark}
Now we can define:
\begin{align}
    \widehat{B}_{ij} &= \sum\limits_{k=1}^n P_{ki}\widehat{X}_{kj}
\end{align}
We can bound the error of this approximation using the triangle inequality:
\begin{align}
\left|B_{ij}-\widehat{B}_{ij}\right| &= \left|\sum\limits_{k=1}^n P_{ki}(\widehat{X}_{kj}-X_{kj})\right| \\
&\leq \sum\limits_{k=1}^n P_{ki}\left|\widehat{X}_{kj}-X_{kj}\right|\\
& \leq \varepsilon\sum\limits_{k=1}^n P_{ki}\\
& = \varepsilon\langle P_{:,i}, 1^n\rangle
\end{align}
Now the problem is calculating $\widehat{B}$. Note that we can write:
\begin{align}
    \widehat{B} = P^T \cdot \widehat{X}
\end{align}
Finally, this takes us back to the calculation of $D^V$. We can use the exact same Markov Chain method and get a final approximation $\widetilde{B}$ so that with probability at least $1-\frac{1}{n}$ it holds that:
\begin{align}
    \left|\widetilde{B}_{ij}-\widehat{B}_{ij}\right| \leq \varepsilon \langle \widehat{X}_{:,j}, 1^n \rangle + 2\varepsilon n M^{(X)}_j
\end{align}
where 
$$
M^{(X)}_j := -\min_{\substack{k\in [n]\\\widehat{X}_{kj}\leq 0}}\widehat{X}_{kj}
$$
Then the overall error can be bounded as follows:
\begin{align}
    \left|\widetilde{B}_{ij}-{B_{ij}}\right| &\leq \left|\widetilde{B}_{ij}-\widehat{B}_{ij}\right| + \left|\widehat{B}_{ij}-{B_{ij}}\right| \\
    &\leq \varepsilon\langle P_{:,i}, 1^n\rangle + \varepsilon \langle \widehat{X}_{:,j}, 1^n \rangle + 2\varepsilon n M^{(X)}_j\\
    &\leq \varepsilon\langle P_{:,i}, 1^n\rangle + \varepsilon\langle {X_{:,j}}, 1^n \rangle + \varepsilon^2 n + 2\varepsilon n M^{(X)}_j\\
    &=\varepsilon\langle P_{:,i}+X_{:,j}, 1^n\rangle + \varepsilon^2 n + 2\varepsilon n M_j^{(X)}
\end{align}
To wrap up our implementation details, we can calculate the required normalization sums as follows:
\begin{align}
\langle\widehat{X}_{:,j}, 1^n \rangle = \sum\limits_{k=1}^n \widehat{X}_{kj} = \sum\limits_{k=1}^n Q_{kj}\widehat{D}_k
\end{align}
We can do this in $\approx \widetilde{O}(dn^{3/2})$ time if we precompute in advance
\begin{align}
    \widehat{D}_k := \langle D^P_{k,:}, P_{k,:}\rangle 
\end{align}
using Lazy Gumbel Sampling for all $k\in[n]$. Further, each element $\widehat{X}_{ij}$ can be computed in $\approx \widetilde{O}(\sqrt{n})$ time as well in a similar fashion. Finally, $M^{(X)}_j$ can also be calculated in such time. Our algorithm is given below as Algorithm \ref{alg:dk_estimation_2}. By combining algorithms \ref{alg:dk_estimation_1} and \ref{alg:dk_estimation_2} we arrive at the following theorem for Algorithm \ref{alg:dk_estimation}:

\begin{algorithm}[h!]
\caption{Estimating $D^K$ -- Part 2: Computing $B$}\label{alg:dk_estimation_2}
\begin{algorithmic}[1]
\State $S_i \gets$ Use an LSH or $k$NN index to calculate $S_i$ for all $i \in [n]$.
\State $\widehat{s} \gets$ 
 \Call{EstimateProductPositive}{$P, 1^n, \varepsilon$}
\vspace{1mm}
\Procedure{Compute--$\widehat{X}_{kj}$}{$Q,K,D^O,V,S_i,\varepsilon,\delta,k,j$}
    \State $F \gets \{Q_{kj}\cdot \langle D^O_{k,:}, V_{s,:}\rangle\}_{s=1}^n \in \mathbb{R}^{n\times 1}$ \Comment{$F$\textit{ will not be materialized.}}
    \State $\widehat{X}_{kj} \gets$ Median-Of-Means with Lazy Gumbel Sampling $\gets Q,K,F, S_i, \varepsilon,\delta$
    \State \Return $\widehat{X}_{kj}$.
\EndProcedure
\vspace{1mm}
\Procedure{Compute--$B$}{$Q,K,V,D^O,\varepsilon$}
    \State Output $\widetilde{B} \gets [0]^{n\times d}$
    \For{$j \in [d]$}
        \Comment{$O(d)$ times}
        \State $\widetilde{B}_{:,j} \gets $\Call{EstimateProduct}{$P, \widehat{X}_{:,j}, \varepsilon, \widehat{s}$}
    \EndFor
    \State \Return $\widetilde{B}$
\EndProcedure
\end{algorithmic}
\end{algorithm}

\begin{theorem}
\label{thm:dk_approximation}
There exists an algorithm that approximates $D^K$ on inputs $Q,K,V,D^O$ under our standard assumptions such that the estimate $\widehat{D}^K$ satisfies:
\begin{align*}
\left|\left|\widehat{D}^K_{:,j}-D^K_{:,j}\right|\right|_\infty \leq \varepsilon\langle P_{:,i}+X_{:,j}, 1^n\rangle &+ \varepsilon^2 n + 2\varepsilon n M_j^{(X)}\\
&+ \varepsilon \sum\limits_{k=1}^n Q_{kj}\cdot \langle(D^o)_{k,:}, V_{i,:}\rangle+ 2\varepsilon n M
\end{align*}
where:
\begin{align}
M\geq M_j^{(i)} := -\min\limits_{\substack{k \in [n]\\Y^{(i)}_{kj} \leq 0}}Y^{(i)}_{kj}\,\text{ and }
M^{(X)}_j := -\min_{\substack{k\in [n]\\\widehat{X_{kj}}\leq 0}}\widehat{X_{kj}}
\end{align}
under our previous definitions for all $j\in[d]$. The algorithm runs in sub-quadratic time and space and succeeds with probability $\geq 1-\delta$.
\end{theorem}
\begin{proof}
We know that $D^K_{ij} = A_{ij} - B_{ij}$. We have that $\widehat{D}^K_{ij} = \widehat{A}_{ij} - \widetilde{B}_{ij}$ and that:
\begin{align}
    |\widehat{A}_{ij} - A_{ij}| &\leq \varepsilon \sum\limits_{k=1}^n Q_{kj}\cdot \langle(D^o)_{k,:}, V_{i,:}\rangle+ 2\varepsilon n M\\
    |\widetilde{B}_{ij} - B_{ij}| &\leq \varepsilon\langle P_{:,i}+X_{:,j}, 1^n\rangle + \varepsilon^2 n + 2\varepsilon n M_j^{(X)}
\end{align}
Thus by the triangle inequality we get the desired error guarantee.
\end{proof}

\begin{algorithm}
\caption{Estimating $D^K$: Putting it all together}\label{alg:dk_estimation}
\begin{algorithmic}
\State $\widetilde{B} \gets$ \Call{Compute--$B$}{$Q,K,V,D^o, \varepsilon$}
\State $\widehat{A} \gets$ \Call{Compute--$A$}{$Q,K,V,D^o,E,\Sigma,\varepsilon$}

\State \Return $\widehat{A} - \widetilde{B}$.

\end{algorithmic}
\end{algorithm}

% \section{Additional Experimental Results for Approximate Gradient Descent}
% We presents experiments complementary to those of Section \ref{sec:grad-approx-experiments} with more convex and non-convex target loss functions. We see that the approximate gradient descent follows the exact one very closely.
% \begin{figure}[!h]
% \centering
% \subfigure[Quadratic $\phi$]{
% \label{fig:efficiency}
% \includegraphics[scale=0.4]{loss_3_quadratic.png}
% }
% \hspace{4pt}
% \subfigure[Cubic $\phi$]{
% \label{fig:error}
% \includegraphics[scale=0.4]{loss_3_cubic.png}
% }
% \label{fig:errors}
% \vspace{-15pt}
% \end{figure}
% \begin{figure}[!h]
%     \centering
%     \includegraphics[width=0.5\linewidth]{loss_2_exponential.png}
%     \caption{Exponential $\phi$}
%     \label{fig:enter-label}
% \end{figure}

\section{Vectorized Implementation of the Forward Pass}
We present the vectorized implementation of $k$NN Attention that we used in our experiments. This is based on Theorem \ref{thm:simpler_expectation}.
\label{section:vectorized_alg_appendix}
\begin{lstlisting}[language=Python, caption=kNN Attention - Vectorized]
import torch

# This function calculates the attention mechanism in the forward pass.
# Inputs:
# - Q: A tensor of shape (b,h,n,d) containing the query vectors.
# - K: A tensor of shape (b,h,n,d) containing the key vectors.
# - V: A tensor of shape (b,h,n,d) containing the value vectors.
# - k: The number of vectors to attend to via kNN methods.
#
# Note that b is the batch size, h is the number of heads, 
# n is the sequence length, and d is the dimension of the vectors.
#
# Outputs:
# - A tensor of shape (b,h,n,d) containing the output vectors.
def attn_forward_batched(Q, K, V, k):
    B,H,N,D = Q.shape

    output = torch.zeros(B,H,N,D, dtype=torch.float32)
    for b in range(B):
        for h in range(H):
            # Get the top k indices of Q[b,h] @ K[b,h]^T and the scores.
            scores, S = topk(Q[b,h,:,:], K[b,h,:,:], k, masking=True)
            scores = scores * (1 / math.sqrt(D))

            # Calculate the denominator.
            M = torch.max(scores, dim=1)[0]
            exp_scores = torch.exp(scores - M.unsqueeze(1))
            denom = torch.sum(exp_scores), dim=1).unsqueeze(1)

            # Calculate the numerator.
            # Vbh[S].shape = (N,k,D)
            numerator = torch.bmm(exp_scores.unsqueeze(1), V[b,h,S])
            numerator.squeeze()

            output[b,h] = numerator / denom

    return output
\end{lstlisting}

\section{Samples from Fine-Tuned GPT2-XL}
\label{appendix:samples}
Below we give some samples from our fine-tuned GPT2-XL using fast approximate kNN attention:

\begin{itemize}
    \item \textbf{Prompt: My favorite thing in the world is: } \textit{Just Cause 3.
    And I couldn't stop playing it.
    It's a mess. I'm not sure how to describe it. But it's just so fun.
    When you play it, you know instantly that this is a very different kind of game, one that takes place in an open world of sorts, and you have a bunch of different characters. It's epic, and it's open world, and you can play as a bounty hunter, or you can play as a doctor, or a cowboy, or you can play as a soldier. 
    You have a buddy system, and you can go to another character's story mission if you want.
    This is a great game, and I'm excited to see what the developers do with it.}
    \item \textbf{Prompt: I like being in nature because}\textit{ it's so natural, and because it's always changing. Nature is a friend and a companion, and I'm so thankful every time I'm able to go out into the outdoors.
    The second thing that I love about being in nature is the comfort and energy that you can feel! I feel so happy and healthy on the trail when I'm interacting with nature. I don't have to think about what I'm wearing, or what I'm doing, and I can just be in nature in my element. I also feel great being out in nature because I'm not spending all day at my desk, working in a office. I get to meet new people along the trail and can spend time talking to them.}
\end{itemize}

\end{document}